\documentclass{article}
\usepackage[margin=1in]{geometry}
\usepackage{bm}
\usepackage{booktabs} 
\usepackage[utf8]{inputenc}
\usepackage{amsmath, amsthm, amssymb}
\usepackage[round]{natbib}
\usepackage{qtree}
\usepackage{graphicx}
\usepackage{tikz}
\usepackage{tkz-graph}
\usepackage{hyperref}
\usepackage{subcaption}
\usetikzlibrary{fit,positioning,arrows,automata}
\usetikzlibrary{bayesnet}
\newtheorem{proposition}{Proposition}
\newtheorem{lemma}{Lemma}

\newtheorem{definition}{Definition}

\newcommand{\neigh}{N} 
\newcommand{\loss}{\mathrm{loss}}
\newcommand{\Ts}{\mathcal{T}}
\newcommand{\Gs}{\mathcal{G}}

\usetikzlibrary{quotes,arrows.meta}
\tikzset{
  annotated cuboid/.pic={
    \tikzset{%
      every edge quotes/.append style={midway, auto},
      /cuboid/.cd,
      #1
    }
    \draw [every edge/.append style={pic actions, densely dashed, opacity=.5}, pic actions]
    (0,0,0) coordinate (o) -- ++(-\cubescale*\cubex,0,0) coordinate (a) -- ++(0,-\cubescale*\cubey,0) coordinate (b) edge coordinate [pos=1] (g) ++(0,0,-\cubescale*\cubez)  -- ++(\cubescale*\cubex,0,0) coordinate (c) -- cycle
    (o) -- ++(0,0,-\cubescale*\cubez) coordinate (d) -- ++(0,-\cubescale*\cubey,0) coordinate (e) edge (g) -- (c) -- cycle
    (o) -- (a) -- ++(0,0,-\cubescale*\cubez) coordinate (f) edge (g) -- (d) -- cycle;
    \path [every edge/.append style={pic actions, |-|}]
    (b) +(0,-5pt) coordinate (b1) 
    (b) +(-5pt,0) coordinate (b2) 
    (c) +(3.5pt,-3.5pt) coordinate (c2) 
    ;
  },
  /cuboid/.search also={/tikz},
  /cuboid/.cd,
  width/.store in=\cubex,
  height/.store in=\cubey,
  depth/.store in=\cubez,
  units/.store in=\cubeunits,
  scale/.store in=\cubescale,
  width=10,
  height=10,
  depth=10,
  units=cm,
  scale=.1,
}

\title{Generalization and Representational Limits of Graph Neural Networks}

\begin{document}

\author{Vikas K. Garg \and Stefanie Jegelka \and Tommi Jaakkola}

\date{CSAIL, MIT}
\maketitle 







\begin{abstract}
\noindent We address two fundamental questions about graph neural networks (GNNs).  First, we prove that several important graph properties cannot be computed by GNNs that rely entirely on local information. Such GNNs include the standard message passing models, and more powerful spatial variants that exploit local graph structure (e.g., via relative orientation of messages, or local port ordering) to distinguish neighbors of each node.  Our treatment includes a novel graph-theoretic formalism. 
Second, we provide the first data dependent generalization bounds for message passing GNNs.  This analysis explicitly accounts for the local permutation invariance of GNNs. Our bounds are much tighter than existing VC-dimension based guarantees for GNNs, and are comparable to Rademacher bounds for recurrent neural networks.
\end{abstract}

\section{Introduction}

Graph neural networks \citep{SGTHM2009, GMS2005}, in their various incarnations, have emerged as models of choice for embedding graph-structured data from a diverse set of domains, including molecular structures, knowledge graphs, biological networks, social networks, and $n$-body problems \citep{DMIBHAA2015, DBV2016, BPLR2016, ZCNC2018, SHBML2018, YJKKK2019, JYBJ2019}. \\

\noindent The working of a graph neural network (GNN) on an input graph, with a feature vector associated with each node, can be outlined as follows. Layer $\ell$ of the GNN updates the embedding of each node $v$ by aggregating the feature vectors, or node and/or edge embeddings, of $v$'s neighbors from layer $\ell-1$ via a non-linear transformation, possibly combining this with $v$'s embedding.
%
The exact form of aggregate and combine steps varies across architectures, and empirical success has been demonstrated for several variants. These include  Graph Convolutional Network (GCN) by \citet{KW2017},   Graph Attention Network (GAT) by \citet{VCCRLB2018}, Graph Isomorphism Network (GIN) by \citet{XHLJ2019}, and GraphSAGE by  \citet{HYL2017}. GNNs are known to have fundamental connections to message passing \citep{DDS2016,GSRVD2017}, the Weisfeiler-Lehman (WL) graph isomorphism test \citep{XHLJ2019, MRFHLRG2019}, and local algorithms \citep{SYK2019, L2019}. \\ 

\noindent In this work, we investigate the representational limitations and generalization properties of GNNs. That is, we examine the performance of GNNs from a learning perspective: (a)~how well they can discriminate graphs that differ in a specified graph property (represented by assigning different labels to such graphs), and (b)~how well they can predict labels, e.g., a graph property, for unseen graphs. 
 Specifically, we focus on classification: (1) a GNN with learnable parameters embeds the nodes of each input graph, (2) the node embeddings are combined into a single graph vector via a {\em readout} function such as sum, average, or element-wise maximum, and (3) a parameterized classifier makes a binary prediction on the resulting graph vector. \\
 
\noindent \textbf{Our contributions.}
(1) First, we show that there exist simple graphs that cannot be distinguished by GNNs  that generate node embeddings solely based on local information.
As a result, these GNNs cannot compute important graph properties such as longest or shortest cycle, diameter, or clique information. This limitation holds for popular models 
such as  GraphSAGE, GCN, GIN, and GAT. Our impossibility results also extend to more powerful variants that provide to each node information about the layout of its neighbors, e.g. via a \emph{port numbering}, like CPNGNN \citep{SYK2019}, or geometric information, like DimeNet \citep{KGG2020}. \\

 
\noindent (2) We introduce a novel graph-theoretic formalism for analyzing CPNGNNs, and our constructions provide insights that may facilitate the design of more effective GNNs. \\
 
\noindent (3) We provide the first data dependent generalization bounds for message passing GNNs. 
Our guarantees are significantly tighter than the VC bounds established by \citet{STH2018} for a class of GNNs. Interestingly, the dependence of our bounds on parameters is comparable to Rademacher bounds for recurrent neural networks (RNNs).
Our results also hold for {\em folding networks} \citep{H2001} that operate on tree-structured inputs. \\

\noindent (4) Our generalization analysis specifically accounts for local permutation invariance of the GNN aggregation function. This relies on a specific sum form that extends to aggregating {\em port-numbered} messages, and therefore opens avenues for analyzing generalization of CPNGNNs.\\

\noindent The rest of the paper is organized as follows. Related work is presented in section \ref{Related}. We provide the necessary background material in section \ref{Preliminaries}. We elucidate the limitations of different GNN variants, introduce our graph-theoretic formalism, and propose a more powerful variant than CPNGNN and DimeNet in section \ref{secLimits}. Finally, we establish generalization bounds for GNNs in section  \ref{secGen}. We outline some key steps of our results in the main text, and defer the details of the proofs to the Supplementary material.     

\section{Related Work} \label{Related}
GNNs continue to generate much interest from both theoretical and practical perspectives. An important theoretical focus has been on understanding the expressivity of existing architectures, and thereby 
introducing richer (invariant) models that can generate more nuanced embeddings. But, much less is known about the generalization ability of GNNs. We briefly review some of these works. \\

\noindent \textbf{Expressivity}. \citet{SGTHM2009} extended the universal approximation property of feed-forward networks (FFNs) \citep{ST1998} to GNNs
using the notion of {\em unfolding equivalence}. Recurrent neural operations for graphs have been introduced with their associated kernel spaces \citep{LJBJ2017}. \citet{DDS2016} performed a sequence of mappings inspired by mean field and  belief propagation procedures from graphical models, and \citet{GSRVD2017} showed that common graph neural net models models may be studied as Message Passing Neural Networks (MPNNs). It is known \citep{XHLJ2019} that GNN variants such as GCNs \citep{KW2017} and GraphSAGE \citep{HYL2017} are no more discriminative than the Weisfeiler-Lehman (WL) test. 
In order to match the power of the WL test, \citet{XHLJ2019} also proposed  GINs.
 Showing GNNs are not powerful enough to represent probabilistic logic inference, \citet{ZCYRLQS2019} 
 introduced {\em ExpressGNN}. Among other works, \citet{BKMPRS2019} proved results  in the context of first order logic, and
\citet{DBY2019} investigated GCNs through the lens of graph moments underscoring the importance of depth compared to width in learning higher order moments. The inability of some graph kernels to distinguish graph properties such as planarity has also been established \citep{KMRS2018, KJM2020}. 
\\ 

 \noindent Spatial, hierarchical, and higher order GNN variants have also been explored. Notably, \citet{SYK2019} exploited a local port ordering of nodes to introduce Consistent Port Numbering GNN (CPNGNN), which they proved to be strictly more powerful than WL. They and \cite{L2019} also established connections to distributed  local algorithms.
 Higher order generalizations have been studied by \citep{MRFHLRG2019, MSRR2019, MFSL2019}; in particular, 
 \citet{Maron2019} introduced models that are more powerful than WL. 
\citet{Hella2015} investigated models weaker than port numbering.  
Several other works exploit spatial information to obtain more nuanced embeddings \citep{YYMRHL2018, YYL2019,  IGBJ2019, KGG2020,CLB2019}. 
\citet{XLTYSKJ2018} learned locally adaptive structure-aware representations by adaptively aggregating information over extended neighborhoods. \citet{VCCRLB2018} introduced GATs that obviate specifying the graph structure in advance. \\

\noindent \textbf{Invariance}. An important consideration in the design of GNNs is their ability to produce output embeddings that are equivariant or permutation-invariant to the input feature vectors. \citet{MBSL2019} constructed
permutation-invariant and equivariant linear layers, and showed that their model can approximate any GNN that can be cast as a MPNN in the framework of \cite{GSRVD2017}. \citet{MSRR2019} constructed new permutation-invariant functions for variable-size inputs, and suggested
some approximations. \citet{MFSL2019, KP2019} proved universality theorems for a specific class of invariant and equivariant networks, respectively.  \\

\noindent \textbf{Generalization}. Several works have established generalization guarantees for FFNs \citep{BFT2017, GRS2018, NBS2018, ZLD2018} and RNNs \citep{CLZ2019, AL2019}. GNNs differ in some key aspects from those models. 
 Unlike RNNs that process sequences, GNNs operate on graph-structured data: sharing of recurrent weights takes place along both the depth and width of a GNN. 
 Unlike FFNs, GNNs deal with irregular local structure. Moreover, at each node, GNNs typically employ permutation-invariant aggregations, in contrast to  global permutation invariance \citep{SGSR2017}. \citet{STH2018} proved VC-dimension bounds for GNNs on a restricted class of graphs that have their label determined by a single designated node. \citet{VZ2019} showed stability bounds for single-layer GCNs in a semi-supervised setting.   

\section{Preliminaries} \label{Preliminaries}
We define the shorthand $[c] = \{1, 2, \ldots, c\}$. For a matrix $W$, we denote its Frobenius norm by $||W||_F$ and spectral norm by $||W||_2$. We also denote the Euclidean norm of a vector $v$ by $||v||_2$. \\

\noindent In a popular class of GNNs, which we call {\em Locally Unordered} GNNs (LU-GNNs), the embedding of each node is updated using messages from its neighbors but without using any spatial information (e.g., the relative orientation of the neighbors). 
This class subsumes variants such as GraphSAGE \citep{HYL2017}, GCN \citep{KW2017}, GIN \citep{XHLJ2019}, and GAT \citep{VCCRLB2018}. 
We can summarize the updated embedding $h_v^{(\ell)}$ for node $v$ at layer $\ell$ in many LU-GNNs by an aggregation and combine operation:
\begin{eqnarray*}\tilde{h}_v^{(\ell-1)} & = & \text{AGG}\{h_u^{(\ell-1)} | u \in \neigh(v)\},\\  h_v^{(\ell)} & = & \text{ COMBINE}\{h_v^{(\ell-1)}, \tilde{h}_v^{(\ell-1)}\}~,
\end{eqnarray*}
where $\neigh(v)$ denotes the set of neighbors of $v$, and functions AGG and COMBINE are sometimes folded into a single aggregation update.  These models are often implemented as MPNNs \citep{DDS2016,GSRVD2017}. One common implementation, called {\em mean field embedding} \citep{DDS2016}, uses the input features $x_v$ of node $v$, in place of $h_v^{(\ell-1)}$ in the COMBINE step above; we will use an instance of this variant for generalization analysis. AGG is typically a permutation-invariant function (e.g., sum). \\ 

\noindent Recently, two subtle variants have been proposed that exploit local structure to treat the neighbors differently.
One of these, CPNGNN \citep{SYK2019}, is based on a {\em consistent port numbering} that numbers the neighbors of each node $v$ from $1 \ldots degree(v)$. Equivalently, a {\em port numbering} (or {\em port ordering}) 
function $p$ associates with each edge $(u, v)$ a pair of numbers $(i, j)$, $i \in [degree(u)]$ and $j \in [degree(v)]$ such that $p(u, i) = (v, j)$, i.e., $u$ is {\em connected} to $v$ via {\em port} $i$. 
%
Thus, $u$ can tell any neighbor from the others based on its ports.
We say $p$ is {\em consistent} if $p(p(u, i)) = (u, i)$ for all $(u, i)$.   Multiple consistent orderings are feasible; CPNGNN arbitrarily fixes one before processing the input graph. 

When computing node embeddings, the 
embedding of node $v$ is updated by processing the information from its neighbors as an \emph{ordered} set, ordered by the port numbering, i.e., the aggregation function is generally not permutation invariant. In addition to a neighbor node $u$'s current embedding, $v$ receives the port number that connects $u$ to $v$.


\noindent Another model, {\em DimeNet} \citep{KGG2020}, is a {\em directional} message passing algorithm introduced in the context of molecular graphs. Specifically, DimeNet embeds atoms via a set of messages (i.e., edge embeddings) and leverages the directional information
by transforming messages based on the angle between them. For each node $v$, the embedding for an incoming message from neighbor $u$ is computed as
\begin{align} \label{equationDime} m_{uv}^{(\ell)} &=  f_1(m_{uv}^{(\ell-1)}, \tilde{m}_{uv}^{(\ell-1)}), ~~~~\quad \text{where}\\ 
\nonumber
 \tilde{m}_{uv}^{(\ell-1)} &= \sum_{w \in \neigh(u)\setminus\{v\}} f_2(m_{wu}^{(\ell-1)}, e^{(uv)}, a^{(wu, uv)}))~,
 \end{align}
and $e^{(uv)}$ is a representation of the distance from $u$ to $v$, $a^{(wu, uv)}$ combines $\angle wuv$ with the distance from $w$ to $u$, and $f_1$ and $f_2$ are update functions 
similar to AGG and COMBINE. The node embedding $h_v^{(\ell)}$ is simply the sum of message embeddings  $m_{uv}^{(\ell)}$.\\

\noindent For a specified graph property $P$ and readout function $f$, we say that a GNN $Q$ {\em decides} $P$, if for any pair of graphs $(G_1, G_2)$ such that $G_1$ and $G_2$ differ on $P$, we have $f(g_Q(G_1)) \neq f(g_Q(G_2))$. Here, $g_Q(G)$ denotes the collection of embeddings of nodes in $G$ when $G$ is provided as input to $Q$. 
We consider several important graph properties in this paper: (a) {\em girth} (length of the shortest cycle), (b) {\em circumference} (length of the longest cycle), 
(c) {\em diameter} (maximum distance, in terms of shortest path, between any pair of nodes in the graph), (d) {\em radius} (minimum node eccentricity, where eccentricity of a node $u$ is defined as the maximum distance from $u$ to other vertices), (e) {\em conjoint cycle} (two cycles that share an edge), (f) {\em total number of cycles}, and (g) {\em $k$-clique} (a subgraph of at least $k \geq 3$ vertices such that each vertex in the subgraph is connected by an edge to any other vertex in the subgraph).


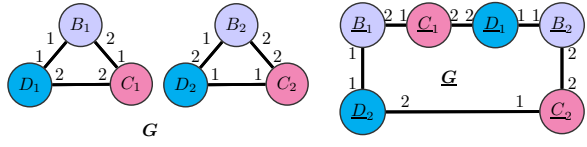
\begin{figure}
\centering
\resizebox{0.47\textwidth}{!}{\begin{tikzpicture}
\tikzstyle{main}=[circle, minimum size = 5mm, thick, draw =black!80, node distance = 6mm]
\tikzstyle{connect}=[-latex, thick]
  \node[main, fill=blue!20] (PB1) {$B_1$};
  \node[main, fill=cyan] (PC1) [below=of PB1, xshift=-1cm, yshift=0.3cm] {$D_1$}; 
  \node[main, fill=magenta!60] (PC2) [below=of PB1, xshift=1cm, yshift=0.3cm] {$C_1$}; 
  \node[main, fill=blue!20] (PB2) [right=of PB1, xshift=1.6cm] {$B_2$};
  \node[main, fill=cyan] (PC3) [below=of PB2, xshift=-1cm, yshift=0.3cm] {$D_2$}; 
  \node[main, fill=magenta!60] (PC4) [below=of PB2, xshift=1cm, yshift=0.3cm] {$C_2$}; 
  

 \draw [ultra thick] (PB1) -- (PC1) node [near start, yshift=0.2cm, xshift=-0.2cm] {1};
 \draw [ultra thick] (PB1) -- (PC1) node [near end, yshift=0.1cm, xshift=-0.2cm] {1};
 \draw [ultra thick] (PB1) -- (PC2) node [near start, yshift=0.2cm, xshift=0.2cm] {2};
 \draw [ultra thick] (PB1) -- (PC2) node [near end, yshift=0.1cm, xshift=0.2cm] {1};
 \draw [ultra thick] (PC1) -- (PC2) node [very near start, yshift=0.2cm] {2};
 \draw [ultra thick] (PC1) -- (PC2) node [very near end, yshift=0.2cm] {2};
 
 \draw [ultra thick] (PB2) -- (PC3) node [near start, yshift=0.2cm, xshift=-0.2cm] {1};
 \draw [ultra thick] (PB2) -- (PC3) node [near end, yshift=0.1cm, xshift=-0.2cm] {2};
 \draw [ultra thick] (PB2) -- (PC4) node [near start, yshift=0.2cm, xshift=0.2cm] {2};
 \draw [ultra thick] (PB2) -- (PC4) node [near end, yshift=0.1cm, xshift=0.2cm] {2};
 \draw [ultra thick] (PC3) -- (PC4) node [very near start, yshift=0.2cm] {1};
 \draw [ultra thick] (PC3) -- (PC4) node [very near end, yshift=0.2cm] {1};
 
 \node[main, fill=blue!20] (PBa) [right=of PB2,  xshift=1cm] {$\underline{B}_1$};
 \node[main, fill=magenta!60] (PCa) [right=of PBa, xshift=-0.2cm] {$\underline{C}_1$};
 \node[main, fill=cyan] (PCb) [right=of PCa, xshift=-0.2cm] {$\underline{D}_1$};
 \node[main, fill=blue!20] (PBb) [right=of PCb, xshift=-0.2cm] {$\underline{B}_2$};
 \node[main, fill=magenta!60] (PCd) [below=of PBb, yshift=-0.2cm] {$\underline{C}_2$};
 \node[main, fill=cyan] (PCc) [below=of PBa, yshift=-0.2cm] {$\underline{D}_2$};
 
 \node[rectangle,draw=white!100, below=of PCc,yshift=2.4cm, xshift=1.7cm] {\bm{$\underline{G}$}};
 
 \draw [ultra thick] (PBa) -- (PCa) node [very near start, yshift=0.2cm] {2};
 \draw [ultra thick] (PBa) -- (PCa) node [very near end, yshift=0.2cm] {1};
 \draw [ultra thick] (PCa) -- (PCb) node [very near start, yshift=0.2cm] {2};
 \draw [ultra thick] (PCa) -- (PCb) node [very near end, yshift=0.2cm] {2};
 \draw [ultra thick] (PCb) -- (PBb) node [very near start, yshift=0.2cm] {1};
 \draw [ultra thick] (PCb) -- (PBb) node [very near end, yshift=0.2cm] {1};
 \draw [ultra thick] (PBb) -- (PCd) node [very near start, xshift=0.2cm] {2};
 \draw [ultra thick] (PBb) -- (PCd) node [very near end, xshift=0.2cm] {2};
 \draw [ultra thick] (PCc) -- (PCd) node [very near start, yshift=0.2cm] {2};
 \draw [ultra thick] (PCc) -- (PCd) node [very near end, yshift=0.2cm] {1};
 \draw [ultra thick] (PCc) -- (PBa) node [very near start, xshift=-0.2cm] {1};
 \draw [ultra thick] (PCc) -- (PBa) node [very near end, xshift=-0.2cm] {1};
\node[rectangle,draw=white!100, below=of PC3, yshift=0.8cm, xshift=-0.7cm] {\bm{$G$}};
\end{tikzpicture}}
\caption{{\bf Construction for Proposition \ref{WLvPN}}. Graph {\boldmath $G$} consists of two triangles that differ in ports (shown next to nodes on each edge) but are otherwise identical, whereas {\boldmath $\underline{G}$} consists of a 6-cycle. LU-GNNs do not use ports,
and each node treats all its  messages equally. Thus, the neighborhood of each node ${X_1}$, where $X \in \{B, C, D\}$ in {\boldmath $G$}, is indistinguishable from that of $\underline{X}_1$ in {\boldmath $\underline{G}$} (so $X_1$ and $\underline{X}_1$ have identical embeddings), and similarly $X_2$ and  $\underline{X}_2$  cannot be told apart. So, LU-GNN with permutation-invariant readout fails to separate {\boldmath $G$} and {\boldmath $\underline{G}$}. In contrast, CPNGNN can exploit that port 2 of $D_2$ connects it to $B_2$, whereas the corresponding node $\underline{D}_2$ connects to $\underline{B}_1$ via port 1.  
\label{fig:LUvS}}
\end{figure}

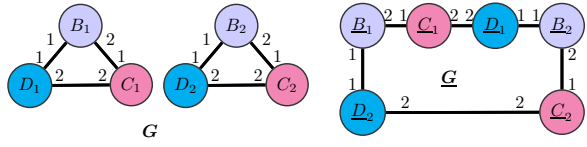
\begin{figure}
\centering
\resizebox{0.47\textwidth}{!}{\begin{tikzpicture}
\tikzstyle{main}=[circle, minimum size = 5mm, thick, draw =black!80, node distance = 6mm]
\tikzstyle{connect}=[-latex, thick]
  \node[main, fill=blue!20] (PB1) {$B_1$};
  \node[main, fill=cyan] (PC1) [below=of PB1, xshift=-1cm, yshift=0.3cm] {$D_1$}; 
  \node[main, fill=magenta!60] (PC2) [below=of PB1, xshift=1cm, yshift=0.3cm] {$C_1$}; 
  
  \node[main, fill=blue!20] (PB2) [right=of PB1, xshift=1.6cm] {$B_2$};
  \node[main, fill=cyan] (PC3) [below=of PB2, xshift=-1cm, yshift=0.3cm] {$D_2$}; 
  \node[main, fill=magenta!60] (PC4) [below=of PB2, xshift=1cm, yshift=0.3cm] {$C_2$}; 
  

 \draw [ultra thick] (PB1) -- (PC1) node [near start, yshift=0.2cm, xshift=-0.2cm] {1};
 \draw [ultra thick] (PB1) -- (PC1) node [near end, yshift=0.1cm, xshift=-0.2cm] {1};
 \draw [ultra thick] (PB1) -- (PC2) node [near start, yshift=0.2cm, xshift=0.2cm] {2};
 \draw [ultra thick] (PB1) -- (PC2) node [near end, yshift=0.1cm, xshift=0.2cm] {1};
 \draw [ultra thick] (PC1) -- (PC2) node [very near start, yshift=0.2cm] {2};
 \draw [ultra thick] (PC1) -- (PC2) node [very near end, yshift=0.2cm] {2};
 
 \draw [ultra thick] (PB2) -- (PC3) node [near start, yshift=0.2cm, xshift=-0.2cm] {1};
 \draw [ultra thick] (PB2) -- (PC3) node [near end, yshift=0.1cm, xshift=-0.2cm] {1};
 \draw [ultra thick] (PB2) -- (PC4) node [near start, yshift=0.2cm, xshift=0.2cm] {2};
 \draw [ultra thick] (PB2) -- (PC4) node [near end, yshift=0.1cm, xshift=0.2cm] {1};
 \draw [ultra thick] (PC3) -- (PC4) node [very near start, yshift=0.2cm] {2};
 \draw [ultra thick] (PC3) -- (PC4) node [very near end, yshift=0.2cm] {2};
 
 
 \node[main, fill=blue!20] (PBa) [right=of PB2,  xshift=1cm] {$\underline{B}_1$};
 \node[main, fill=magenta!60] (PCa) [right=of PBa, xshift=-0.2cm] {$\underline{C}_1$};
 \node[main, fill=cyan] (PCb) [right=of PCa, xshift=-0.2cm] {$\underline{D}_1$};
 \node[main, fill=blue!20] (PBb) [right=of PCb, xshift=-0.2cm] {$\underline{B}_2$};
 \node[main, fill=magenta!60] (PCd) [below=of PBb, yshift=-0.2cm] {$\underline{C}_2$};
 \node[main, fill=cyan] (PCc) [below=of PBa, yshift=-0.2cm] {$\underline{D}_2$};
 
 \node[rectangle,draw=white!100, below=of PCc,yshift=2.4cm, xshift=1.7cm] {\bm{$\underline{G}$}};
 
 \draw [ultra thick] (PBa) -- (PCa) node [very near start, yshift=0.2cm] {2};
 \draw [ultra thick] (PBa) -- (PCa) node [very near end, yshift=0.2cm] {1};
 \draw [ultra thick] (PCa) -- (PCb) node [very near start, yshift=0.2cm] {2};
 \draw [ultra thick] (PCa) -- (PCb) node [very near end, yshift=0.2cm] {2};
 \draw [ultra thick] (PCb) -- (PBb) node [very near start, yshift=0.2cm] {1};
 \draw [ultra thick] (PCb) -- (PBb) node [very near end, yshift=0.2cm] {1};
 \draw [ultra thick] (PBb) -- (PCd) node [very near start, xshift=0.2cm] {2};
 \draw [ultra thick] (PBb) -- (PCd) node [very near end, xshift=0.2cm] {1};
 \draw [ultra thick] (PCc) -- (PCd) node [very near start, yshift=0.2cm] {2};
 \draw [ultra thick] (PCc) -- (PCd) node [very near end, yshift=0.2cm] {2};
 \draw [ultra thick] (PCc) -- (PBa) node [very near start, xshift=-0.2cm] {1};
 \draw [ultra thick] (PCc) -- (PBa) node [very near end, xshift=-0.2cm] {1};
\node[rectangle,draw=white!100, below=of PC3, yshift=0.8cm, xshift=-0.7cm] {\bm{$G$}};
\end{tikzpicture}}
\caption{{\bf Construction for Proposition \ref{PortIssues}}. Graphs {\boldmath $G$} and 
{\boldmath $\underline{G}$} are same as in Fig. \ref{fig:LUvS}, but have been assigned a different consistent numbering. CPNGNN can no longer distinguish the graphs with permutation-invariant readout since each node labeled with $X_1$ in {\boldmath $G$}, where $X \in \{B, C, D\}$ has a corresponding node labeled $\underline{X}_1$ in  {\boldmath $\underline{G}$} with identical features and indistinguishable port-numbered neighborhoods (similarly for $X_2$).  Thus, ordering matters.    
\label{fig:PortDisc}}
\end{figure}

\section{Representation limits of GNNs} \label{secLimits}
We now sketch novel constructions to illustrate the limits of LU-GNNs, CPNGNNs, and DimeNets. First, we show that in some cases, CPNGNNs can be more discriminative than LU-GNNs, depending on the port numbering. Then, we demonstrate that still, LU-GNNs, CPNGNNs, and DimeNets cannot compute certain graph properties.
Our proofs build examples of graphs that (1) differ in important graph properties, but that (2) these models cannot distinguish. As a consequence, these models will not be able to compute such graph properties in general.\\

\noindent To formalize this framework, we introduce a condition of local isomorphism for a pair of graphs. This condition implies that CPNGNNs and LU-GNNs cannot distinguish the two graphs. A similar framework applies to DimeNet. Finally, our insights point to a new GNN variant that leverages additional geometric features to circumvent our constructions for CPNGNNs and DimeNets.


\paragraph{Limitations of LU-GNNs.} 

\begin{proposition} \label{WLvPN}
There exist some graphs that LU-GNNs cannot distinguish, but CPNGNN can distinguish with some consistent port ordering.  
\end{proposition} 

\begin{figure*}
\begin{subfigure}[b]{0.5\textwidth}
 \resizebox{0.95\textwidth}{!}{\begin{tikzpicture}
\tikzstyle{main}=[circle, minimum size = 5mm, thick, draw =black!80, node distance = 6mm]
\tikzstyle{connect}=[-latex, thick]
  \node[main, fill=red!20] [xshift=2cm] (PA) {$A_1$};
  \node[main, fill=blue!20] (PB) [right=of PA] {$B_1$};
  \node[main, fill=magenta!60] (PC) [below=of PB] {$C_1$};
 \node[main, fill=cyan] (PD) [below=of PA] {$D_1$};
 
 
 \node[main, fill=red!20] [right=of PB] (PA2) {$A_2$};
  \node[main, fill=blue!20] (PB2) [right=of PA2] {$B_2$};
  \node[main, fill=magenta!60] (PC2) [below=of PB2] {$C_2$};
 \node[main, fill=cyan] (PD2) [below=of PA2] {$D_2$};

 \node[main, fill=red!20] (L1) [right=of PB2, xshift=1cm] {$\underline{A}_1$};
  \node[main, fill=blue!20] (L2) [right=of L1] {$\underline{B}_1$};
  \node[main, fill=magenta!60] (L3) [right=of L2] {$\underline{C}_1$};
  \node[main, fill=cyan] (Lt) [right=of L3] {$\underline{D}_1$};
  \node[main,fill=cyan] (O1) [below=of L1] {$\underline{D}_2$};
  \node[main,fill=magenta!60] (O2) [right=of O1,below=of L2] {$\underline{C}_2$};
  \node[main,fill=blue!20] (O3) [right=of O2,below=of L3] {$\underline{B}_2$};
  \node[main,fill=red!20] (Ot) [right=of O3,below=of Lt] {$\underline{A}_2$};

\draw [ultra thick] (PA) -- (PB) node [very near start, yshift=0.2cm] {1};
 \draw [ultra thick] (PA) -- (PB) node [very near end, yshift=0.2cm] {1}; 
 
 \draw [ultra thick] (PB) -- (PC) node [very near start, xshift=0.2cm] {2};
 \draw [ultra thick] (PB) -- (PC) node [very near end, xshift=0.2cm, yshift=0.05cm] {2}; 
 \draw [ultra thick] (PC) -- (PD) node [very near start, yshift=0.2cm] {1};
 \draw [ultra thick] (PC) -- (PD) node [very near end, yshift=0.2cm] {1}; 
 
 \draw [ultra thick] (PD) -- (PA) node [very near start, xshift=-0.2cm, yshift=1pt] {2};
 \draw [ultra thick] (PD) -- (PA) node [very near end, xshift=-0.2cm] {2};
 
 \draw [ultra thick] (PA2) -- (PB2) node [very near start, yshift=0.2cm] {1};
 \draw [ultra thick] (PA2) -- (PB2) node [very near end, yshift=0.2cm] {1}; 
 
 \draw [ultra thick] (PB2) -- (PC2) node [very near start, xshift=0.2cm] {2};
 \draw [ultra thick] (PB2) -- (PC2) node [very near end, xshift=0.2cm, yshift=0.05cm] {2}; 
 \draw [ultra thick] (PC2) -- (PD2) node [very near start, yshift=0.2cm] {1};
 \draw [ultra thick] (PC2) -- (PD2) node [very near end, yshift=0.2cm] {1}; 
 
 \draw [ultra thick] (PD2) -- (PA2) node [very near start, xshift=-0.2cm, yshift=1pt] {2};
 \draw [ultra thick] (PD2) -- (PA2) node [very near end, xshift=-0.2cm] {2};
 
 \draw [ultra thick] (L1) -- (L2) node [very near start, yshift=0.2cm] {1};
 \draw [ultra thick] (L1) -- (L2) node [very near end, yshift=0.2cm] {1};
 
 \draw [ultra thick] (L2) -- (L3) node [very near start, yshift=0.2cm] {2};
 \draw [ultra thick] (L2) -- (L3) node [very near end, yshift=0.2cm] {2};
 
 \draw [ultra thick] (L3) -- (Lt) node [very near start, yshift=0.2cm] {1};
 \draw [ultra thick] (L3) -- (Lt) node [very near end, yshift=0.2cm] {1};
 
 \draw [ultra thick] (O1) -- (O2) node [very near start, yshift=0.2cm] {1};
 \draw [ultra thick] (O1) -- (O2) node [very near end, yshift=0.2cm] {1};
 
 \draw [ultra thick] (O2) -- (O3) node [very near start, yshift=0.2cm] {2};
 \draw [ultra thick] (O2) -- (O3) node [very near end, yshift=0.2cm] {2};
 
 \draw [ultra thick] (O3) -- (Ot) node [very near start, yshift=0.2cm] {1};
 \draw [ultra thick] (O3) -- (Ot) node [very near end, yshift=0.2cm] {1};
 
 \draw [ultra thick] (O1) -- (L1) node [very near start, xshift=-0.2cm, yshift=1pt] {2};
 \draw [ultra thick] (O1) -- (L1) node [very near end, xshift=-0.2cm] {2};
 
 \draw [ultra thick] (Ot) -- (Lt) node [very near start, xshift=0.2cm, yshift=1pt] {2};
 \draw [ultra thick] (Ot) -- (Lt) node [very near end, xshift=0.2cm] {2};

        \foreach \from/\to in {L1/O1, L1/L2, L2/L3, L3/Lt, Lt/Ot, O1/O2, O2/O3, O3/Ot, PA/PB, PB/PC, PC/PD, PD/PA,
        PA2/PB2, PB2/PC2, PC2/PD2, PD2/PA2}
 \draw [ultra thick] (\from) -- (\to); 
         \node[rectangle,draw=white!100, below=of PA, yshift=1cm, xshift=0.7cm] {\bm{$S_4$}};
         \node[rectangle,draw=white!100, below=of PA2, yshift=1cm, xshift=0.7cm] {\bm {$S_4$}};
\node[rectangle,draw=white!100, below=of L1, yshift=1cm, xshift=2cm] {\bm{$S_8$}};
    
\end{tikzpicture}}
\end{subfigure}
~~
\begin{subfigure}[b]{0.5\textwidth}
 \resizebox{0.9\textwidth}{!}{\begin{tikzpicture}
\tikzstyle{main}=[circle, minimum size = 5mm, thick, draw =black!80, node distance = 7mm]
\tikzstyle{connect}=[-latex, thick]
  \node[main, fill=red!20] [xshift=2cm, yshift=1cm] (PA) {$A_1$};
  \node[main, fill=blue!20] (PB) [right=of PA, yshift=-1cm] {$B_1$};
  \node[main, fill=magenta!60] (PC) [left=of PB, yshift=-1cm] {$C_1$};
 \node[main, fill=cyan] (PD) [left=of PA, yshift=-1cm] {$D_1$};
 
 \node[main, fill=red!20] [right=of PA, xshift=2.7cm] (PA2) {$A_2$};
  \node[main, fill=blue!20] (PB2) [right=of PA2, yshift=-1cm] {$B_2$};
  \node[main, fill=magenta!60] (PC2) [left=of PB2, yshift=-1cm] {$C_2$};
 \node[main, fill=cyan] (PD2) [left=of PA2, yshift=-1cm] {$D_2$};

\draw [ultra thick] (PA) -- (PB) node [very near start, xshift=0.2cm, yshift=1pt] {1};
 \draw [ultra thick] (PA) -- (PB) node [very near end, yshift=0.2cm] {1};
 
 \draw [ultra thick] (PA2) -- (PB2) node [very near start, xshift=0.2cm, yshift=1pt] {1};
 \draw [ultra thick] (PA2) -- (PB2) node [very near end, yshift=0.2cm] {1};
 
 \draw [ultra thick] (PB) -- (PC) node [very near start, xshift=-0.2cm, yshift=2pt] {2};
 \draw [ultra thick] (PB) -- (PC) node [very near end, yshift=0.2cm] {2};
 
 \draw [ultra thick] (PB2) -- (PC2) node [very near start, xshift=-0.2cm, yshift=2pt] {2};
 \draw [ultra thick] (PB2) -- (PC2) node [very near end, yshift=0.2cm] {2};
 
 \draw [ultra thick] (PC) -- (PD) node [very near start, xshift=0.1cm, yshift=4pt] {1};
 \draw [ultra thick] (PC) -- (PD) node [very near end, yshift=0.1cm, xshift=0.1cm] {1};
 
 \draw [ultra thick] (PC2) -- (PD2) node [very near start, xshift=0.1cm, yshift=4pt] {1};
 \draw [ultra thick] (PC2) -- (PD2) node [very near end, yshift=0.1cm, xshift=0.1cm] {1};

\draw [ultra thick] (PD) -- (PA) node [very near start, xshift=-0.1cm, yshift=5pt] {2};
 \draw [ultra thick] (PD) -- (PA) node [very near end, xshift=-0.2cm, yshift=0.1cm] {2};
 
 \draw [ultra thick] (PD2) -- (PA2) node [very near start, xshift=-0.1cm, yshift=5pt] {2};
 \draw [ultra thick] (PD2) -- (PA2) node [very near end, xshift=-0.2cm, yshift=0.1cm] {2};
 
 \draw [ultra thick] (PD) -- (PB) node [very near start, xshift=-0.1cm, yshift=5pt] {3};
 \draw [ultra thick] (PD) -- (PB) node [very near end, xshift=-0.05cm, yshift=0.2cm] {3};
 
  \draw [ultra thick] (PD2) -- (PB2) node [very near start, xshift=-0.1cm, yshift=5pt] {3};
 \draw [ultra thick] (PD2) -- (PB2) node [very near end, xshift=-0.05cm, yshift=0.2cm] {3};

   
   \node[main, fill=red!20] [xshift=6cm, yshift=0.5cm, left=of PA2] (RA) {$\underline{A}_1$};
  \node[main, fill=blue!20] (RB) [right=of RA, yshift=-1cm] {$\underline{B}_1$};
  \node[main, fill=magenta!60] (RC) [left=of RB] {$\underline{C}_1$};
 \node[main, fill=cyan] (RD) [left=of RA, yshift=-1cm] {$\underline{D}_1$};

\draw [ultra thick] (RA) -- (RB) node [very near start, xshift=0.2cm, yshift=1pt] {1};
 \draw [ultra thick] (RA) -- (RB) node [very near end, yshift=0.2cm] {1};
 
 \draw [ultra thick] (RC) -- (RD) node [very near start, xshift=-0.1cm, yshift=0.2cm] {1};
 \draw [ultra thick] (RC) -- (RD) node [near end, yshift=0.2cm] {1};
 
 \draw [ultra thick] (RB) -- (RC) node [very near start, xshift=-0.1cm, yshift=0.2cm] {2};
 \draw [ultra thick] (RB) -- (RC) node [very near end, yshift=0.2cm] {2};
 
 \draw [ultra thick] (RD) -- (RA) node [very near start, xshift=-0.1cm, yshift=0.2cm] {2};
 \draw [ultra thick] (RD) -- (RA) node [very near end, yshift=0.2cm, xshift=-0.1cm] {2};

 \foreach \from/\to in {RA/RB, RB/RC, RC/RD, RD/RA}
 \draw [ultra thick] (\from) -- (\to); 
  
 \node[rectangle,draw=white!100, below=of PA, yshift=0.9cm] {\bm{$G_1$}};
 \node[rectangle,draw=white!100, below=of PA2, yshift=0.9cm] {\bm{$G_1$}};

  \node[main, fill=red!20] [below=of RC, yshift=0.5cm] (BA) {$\underline{A}_2$};
  \node[main, fill=cyan] (BD) [right=of BA, yshift=-1cm] {$\underline{D}_2$};
  \node[main, fill=magenta!60] (BC) [left=of BD] {$\underline{C}_2$};
 \node[main, fill=blue!20] (BB) [left=of BA, yshift=-1cm] {$\underline{B}_2$};

\draw [ultra thick] (BA) -- (BD) node [very near start, xshift=0.2cm, yshift=1pt] {1};
 \draw [ultra thick] (BA) -- (BD) node [very near end, yshift=0.2cm] {1};
 
 \draw [ultra thick] (BC) -- (BB) node [very near start, xshift=-0.1cm, yshift=0.2cm] {1};
 \draw [ultra thick] (BC) -- (BB) node [near end, yshift=0.2cm] {1};
 
 \draw [ultra thick] (BD) -- (BC) node [very near start, xshift=-0.1cm, yshift=0.2cm] {2};
 \draw [ultra thick] (BD) -- (BC) node [very near end, yshift=0.2cm] {2};
 
 \draw [ultra thick] (BB) -- (BA) node [very near start, xshift=-0.1cm, yshift=0.2cm] {2};
 \draw [ultra thick] (BB) -- (BA) node [very near end, yshift=0.2cm, xshift=-0.1cm] {2};
 
 \draw [ultra thick] (BB) -- (RD) node [very near start, xshift=-0.2cm] {3};
 \draw [ultra thick] (BB) -- (RD) node [very near end, yshift=-0.1cm, xshift=-0.2cm] {3};
 
 \draw [ultra thick] (BD) -- (RB) node [very near start, xshift=-0.2cm] {3};
 \draw [ultra thick] (BD) -- (RB) node [very near end, yshift=-0.1cm, xshift=-0.2cm] {3};

\node[rectangle,draw=white!100, below=of RD, yshift=1cm, xshift=0.8cm] {\bm{$G_2$}};
  

\end{tikzpicture}}
\end{subfigure}
\caption{{\bf Constructions for Proposition \ref{prop1}}. The graph with two copies of {\boldmath $S_4$} is indistinguishable from {\boldmath $S_8$}  despite having different girth, circumference, diameter, radius, and total number of cycles.
This follows since for each $X \in \{A, B, C, D\}$, nodes $X_1$ and $\underline{X}_1$ have identical feature vectors as well as identical port-ordered neighborhoods (similarly for nodes $X_2$ and $\underline{X}_2$).      
Likewise, the graph with two copies of {\boldmath $G_1$}, each having a conjoint cycle, cannot be distinguished from {\boldmath $G_2$} as the graphs are port-locally isomorphic. A simple modification extends the result to $k$-clique (described in the Supplementary). The constructions hold for LU-GNNs as well (by simply ignoring the port numbers). Note that, in contrast, DimeNet is able to distinguish the graphs in these constructions, e.g., using that $\angle A_1B_1C_1$ is different from the corresponding $\angle \underline{A}_1\underline{B
}_1\underline{C}_1$.
\label{fig:prop3}}
\end{figure*}
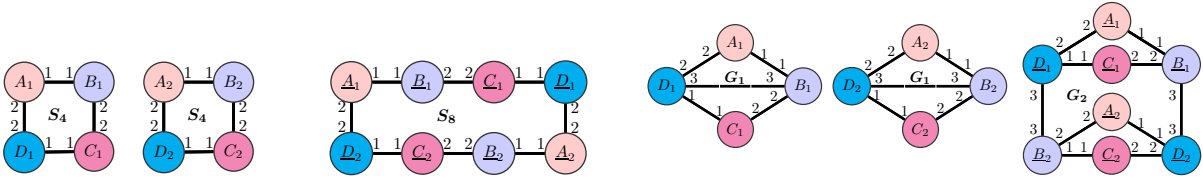

\noindent Fig. \ref{fig:LUvS} shows two graphs, {\boldmath $G$} (consisting of two triangles) and {\boldmath $\underline{G}$}. Nodes with same color (or, equivalently, same uppercase symbol without the subscripts and underline) have identical feature vectors. The port numbers for each node are shown next to the node on the respective edges (note the numbering is consistent). Moreover, for LU-GNNs, edges on nodes with same color have identical edge feature vectors; for CPNGNNs, edge features are the same if, in addition, the local ports for nodes that have the same color are identical. As explained in Fig. \ref{fig:LUvS}, CPNGNN 
can distinguish between the two graphs by exploiting the port information. However, LU-GNNs do not leverage such information, and fail to find distinct representations.\\ 


\paragraph{Limitations of CPNGNNs.} 
Port orderings help CPNGNNs distinguish graphs better. But, port orderings are not unique, and not all orderings distinguish the same set of graphs.
\begin{proposition} \label{PortIssues}
There exist pairs of graphs and consistent port numberings $p$ and $q$ such that CPNGNN can distinguish the graphs with $p$ but not $q$.
\end{proposition} 

\noindent Fig. \ref{fig:PortDisc} shows the same pair of graphs {\boldmath $G$} and {\boldmath $\underline{G}$} but with a different ordering. 
CPNGNN can no longer distinguish the two non-isomorphic graphs with this new ordering. Therefore, 
it may be useful to try multiple random orderings, or even parameterize and learn one along with other GNN parameters. \\

\noindent Henceforth, we assume that an ordering is given with the input graph. We now demonstrate the inability of CPNGNNs to decide several graph properties.  
Toward that goal,
note that in Fig. \ref{fig:LUvS} we conjured an explicit bijection between nodes in {\boldmath $G$} and {\boldmath $\underline{G}$} to reason about permutation-invariant readouts. We now introduce a graph-theoretic  formalism for CPNGNNs that obviates the need for an explicit bijection and is easier to check.  \\

\noindent 
We define a pair of surjective mappings between two graphs in question, 
and impose additional conditions that guarantee the existence of a bijection. This bijection implies that corresponding nodes in the graphs receive identical embeddings, and hence both graphs obtain the same set of node embeddings, making them indistinguishable. \\

\noindent The main idea is that a node $v_1$ in graph $G_1$ is locally indistinguishable from $v_2$ in $G_2$ if (1) the node features agree: $x_{v_1} = x_{v_2}$, and (2) the port-ordered local neighborhoods of $v_1$ and $v_2$ cannot be told apart. That is, if port $i$ of $v_1$ connects to port $k$ of $v$, then a locality preserving bijection connects the nodes corresponding to images of $v_1$ and $v$ via the same ports.
In the notation here, we include the port numbers $(i,j)$ associated with each edge $(u, v)$ in the edge notation, i.e., $((u, i), (v, j))$.


\begin{definition} \label{Def1}
\rm We say that graph $G_1(V_1, E_1, p)$ {\em port-covers} $G_2(V_2, E_2, q)$ if the following conditions are satisfied:
{\bf (a)} there exists a surjection $f : V_1 \mapsto V_2$ such that $x_{v} = x_{f(v)}$ for all $v \in V_1$, 
{\bf (b)} $p$ and $q$ are consistent, and 
{\bf (c)} for all $v_1 \in V_1$ there exists a local bijection $g_{v_1}$ such that for all $i \in [degree(v_1)]$ and $(v, k) = p(v_i, i)$, we have   
$$g_{v_1}(((v_1, i), (v, k))) = (
q(f(v), k), q(f(v_1), i))~,$$
such that $q(f(v), k) = (f(v_1), i)$; $q(f(v_1), i) = (f(v), k)$;
$((v_1, i), (v, k)) \in E_1$; and $(
q(f(v), k), q(f(v_1), i)) \in E_2$.
\noindent Moreover, we say that $G_1(V_1, E_1, p)$ and $G_2 (V_2, E_2, q)$ are {\em port-locally isomorphic} if they both cover each other. 
\end{definition}

\noindent Note that Definition \ref{Def1} does not preclude the possibility that $f$ maps multiple nodes in $G_1$ to the same node in $G_2$, or the other way round. That being the case, the claim that $G_1$ and $G_2$ cannot be distinguished by CPNGNN might not hold. Fortunately, the following result comes to our rescue. 
\begin{proposition} \label{Bijection}
If $G_1 (V_1, E_1, p)$ and $G_2(V_2, E_2, q)$ are port-locally isomorphic, there exists a bijection $h$ that satisfies {\bf (a)}-{\bf(c)} in Definition \ref{Def1} (with $h$ replacing $f$).  As a corollary, CPNGNNs produce identical embeddings for the corresponding nodes in $G_1$ and $G_2$, so CPNGNNs cannot separate $G_1$ and $G_2$ with  permutation-invariant readout.    
\end{proposition}

\noindent We now proceed to establish that CPNGNNs are limited in that they fail to decide important graph properties. We can invoke conditions of Proposition \ref{Bijection}, or define a bijection, to show the following result (see Fig. \ref{fig:prop3} for our constructions). 

\begin{proposition} \label{prop1}
There exist consistent port orderings such that CPNGNNs with permutation-invariant readout cannot decide several important graph properties such as girth,   circumference, diameter, radius, conjoint cycle, total number of cycles, and $k$-clique.
\end{proposition}

\noindent Clearly, these impossibility results apply to LU-GNNs as well (see Fig. \ref{fig:prop3}). However, as described in Fig. \ref{fig:prop3},  our constructions for CPNGNNs do not work for DimeNets. This immediately leads us to the question whether DimeNets are expressive enough to decide the graph properties. 

\paragraph{Limitations of DimeNets.} Unfortunately, as we show in Fig. \ref{fig:Dime}, it turns out we can craft another construction on two graphs that differ in several of these properties but cannot be distinguished by DimeNets.     
\begin{proposition} \label{propDime1}
DimeNet with  permutation-invariant readout cannot decide graph properties such as girth, circumference, diameter,  radius, or total number of cycles.   
\end{proposition}
\noindent In fact, as we argue in Fig. \ref{fig:Dime}, augmenting DimeNet with port-numbering would still not be sufficient. Therefore, a natural question that arises is whether we can obtain a more expressive model than both CPNGNN and DimeNet. 
Leveraging insights from our constructions, we now introduce one such variant, {\em H-DCPN} (short for {\em Hierarchical Directional Message Passing Consistent Port Numbering Networks}), that generalizes both CPNGNN and DimeNet. 

\paragraph{More powerful GNNs.} 
 The main idea is to augment DimeNet not just with port ordering, but also additional spatial information. Observe that the construction in Fig.~\ref{fig:Dime} will fail if for each edge $(u,v)$, we additionally model the set of angles $\alpha_{wuvz}$ between planes $\mathcal{P}(w, u, v)$ and $\mathcal{P}(u, v, z)$ due to neighbors $w$ of $u$ and neighbors $z$ of $v$. Similarly, we could use the distances between these planes. We denote by $\Phi_{uv}$ all such features due to these planes. Denote by $m_{uv}^{(\ell)}$ the message from neighbor $u$ of $v$ at time $\ell$, and by $\underline{m}_{uv}^{(\ell)} = \underline{f}(m_{uv}^{(\ell)},\Phi_{uv})$ a refined message that encapsulates the effect of geometric features.  \\

\noindent We incorporate salient aspects of CPNGNN as well. Specifically, we first fix a consistent port numbering, as in CPNGNN.  Denote the degree of $v$ by $d(v)$. Let $c_v(j)$ be the neighbor of $v$ that connects to port $j$ of $v$ via port $t_{j, v}$, 
for $j \in [d(v)]$.  
We suggest to update the embedding of $v$ as
\begin{equation*} \label{H-DCPN}
h_v^{(\ell)} = f(h_v^{(\ell-1)}, \underline{m}_{c_v(1)v}^{(\ell-1)}, t_{1,v}, \ldots, \underline{m}_{c_v(d(v))v}^{(\ell-1)}, t_{d(v),v})~, \end{equation*}
where $f$ can potentially take into account the ordering of its arguments. The update resembles CPNGNN when we define $m_{uv}^{(\ell)} = h_u^{(\ell)}$; and DimeNet when $f$ ignores $h_v^{(\ell-1)}$ (and ports) and we  define $m_{uv}^{(\ell)}$ using \eqref{equationDime} in section \ref{Preliminaries}. H-DCPN derives its additional discriminative power from the features $\Phi_{uv}$ encoded in messages $\underline{m}_{uv}^{(\ell)}$. For instance, the nodes labeled   
$A_1$, $B_1$, $C_1$, $D_1$ lie on the same plane in {\boldmath ${G}_3$}. In contrast, the plane defined by nodes with labels $\underline{A}_1$, $\underline{B}_1$, $\underline{C}_1$ in {\boldmath $G_4$} is orthogonal to that defined by nodes with labels $\underline{D}_2, \underline{A}_1, \underline{B}_1$; thus allowing H-DCPN to distinguish the node labeled $A_1$ from the node labeled $\underline{A}_1$ (Fig. \ref{fig:Dime}). 

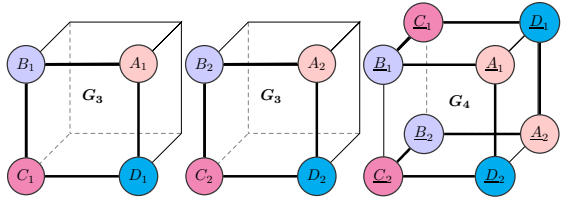
\begin{figure}[!t]
\centering
\begin{subfigure}[b]{0.5\textwidth}
 \resizebox{0.9\textwidth}{!}{\begin{tikzpicture}
\tikzstyle{main}=[circle, minimum size = 3mm, thick, draw =black!80]
  \pic at (1,-3) {annotated cuboid={width=250, height=250, depth=250, scale=.01, units=}};
  \pic at (5,-3) {annotated cuboid={width=250, height=250, depth=250, scale=.01, units=}};
   \pic at (9,-3) {annotated cuboid={width=250, height=250, depth=250, scale=.01, units=}};
   
 \node[main, fill=cyan] (PD)[yshift=-5.5cm, xshift=1cm] {$D_1$};
 \node[main, fill=red!20]  (PA)[above=of PD, yshift=0.65cm] {$A_1$};
 \node[main,fill=magenta!60]  (PC)[left=of PD, xshift=-0.65cm] {$C_1$};
\node[main, fill=blue!20] (PB) [above=of PC, yshift=0.65cm] {$B_1$};
\foreach \from/\to in {PA/PB, PB/PC, PC/PD, PD/PA}
 \draw[ultra thick] (\from) -- (\to); 
 
 \node[main, fill=cyan] (RD)[yshift=-5.5cm, xshift=5cm] {$D_2$};
 \node[main, fill=red!20]  (RA)[above=of RD, yshift=0.65cm] {$A_2$};
 \node[main,fill=magenta!60]  (RC)[left=of RD, xshift=-0.65cm] {$C_2$};
\node[main, fill=blue!20] (RB) [above=of RC, yshift=0.65cm] {$B_2$};
\foreach \from/\to in {RA/RB, RB/RC, RC/RD, RD/RA}
 \draw[ultra thick] (\from) -- (\to);

\node[main, fill=cyan] (PD2)[yshift=-5.5cm, xshift=9cm] {$\underline{D}_2$}; 
\node[main, fill=red!20]  (PA1)[above=of PD2, yshift=0.6cm] {$\underline{A}_1$};
 \node[main,fill=magenta!60]  (PC2)[left=of PD2, xshift=-0.6cm] {$\underline{C}_2$};
 \node[main, fill=blue!20] (PB1) [above=of PC2, yshift=0.6cm] {$\underline{B}_1$};
 \foreach \from/\to in {PA1/PB1,  PC2/PD2, PD2/PA1}
 \draw[ultra thick] (\from) -- (\to); 
 
 \node[main, fill=red!20] (PA2)[yshift=-4.55cm, xshift=10cm] {$\underline{A}_2$}; 
 \node[main, fill=blue!20] (PB2) [left=of PA2, xshift=-0.7cm] {$\underline{B}_2$};
 \node[main, fill=cyan] (PD1) [above=of PA2, yshift=0.6cm] {$\underline{D}_1$};
 \node[main,fill=magenta!60] (PC1) [left=of PD1, xshift=-0.7cm] {$\underline{C}_1$};
 \foreach \from/\to in {  PB1/PC1, PC1/PD1, PB2/PA2, PD1/PA2, PC2/PB2}
 \draw[ultra thick] (\from) -- (\to); 

\node[draw=white!100, above=of PC, yshift=0.1cm, xshift=1.5cm] {\bm{$G_3$}};

\node[draw=white!100, above=of RC, yshift=0.1cm, xshift=1.5cm] {\bm{$G_3$}};

\node[draw=white!100, above=of PB2, yshift=-1cm, xshift=0.8cm] {\bm{$G_4$}};

\end{tikzpicture}}
\end{subfigure}
\caption{{\bf Construction for Proposition \ref{propDime1}}. DimeNet cannot discriminate between \boldmath{$G_4$} and the other graph
that consists of two identical copies of \boldmath${G_3}$, since the corresponding local angles and distances are identical in the two graphs.
Moreover, since $G_3$ and $G_4$ are obtained by overlaying $S_4$ and $S_8$ (from Fig. \ref{fig:prop3}), respectively on a cube, augmenting DimeNet with the port-numbering scheme from $S_4$ and $S_8$ will still not be sufficient to distinguish the graphs. 
\label{fig:Dime}}
\end{figure}

\section{Generalization bounds for GNNs} \label{secGen} 
Next, we study the generalization ability of GNNs via Rademacher bounds, focusing on binary classification. We  generalize the previous results on the complexity of feedforward networks \citep{BFT2017, NBS2018} and RNNs  \citep{CLZ2019} in mainly three ways.  First, we process graphs unlike sequences in RNNs, or instances restricted to the input layer in feedforward networks. Toward that goal, we show the complexity of GNNs that combine predictions from individual nodes may be analyzed by focusing on local node-wise computation trees.  Second, we share weights across all nodes in these computation trees (i.e., both along the depth and the width of the  tree). Third, we model {\em local} permutation-invariance in the aggregate function at each node in the tree. Our bounds are comparable to the Rademacher bounds for RNNs. \\
%

\noindent We consider locally permutation invariant GNNs, where in each layer $\ell$, the embedding $h_v^{\ell} \in \mathbb{R}^r$ of node $v$ of a given input graph is updated by aggregating the embeddings of its neighbors, $u \in \neigh(v)$, via an aggregation function $\rho : \mathbb{R}^r \to \mathbb{R}^r$. Different types of updates are possible; we focus on a {\em mean field} update \citep{DDS2016, JBJ2018, JYBJ2019}:
\begin{equation} \label{SumDecomp}
  h^\ell_v = \phi\big(W_1 x_v + W_2 \rho( \sum\nolimits_{u \in \neigh(v)} g(h^{\ell-1}_u))\big),
\end{equation}
where $\phi$ and $g$ are nonlinear transformations and $x_v \in \mathbb{R}^r$ is the feature vector for $v$. We assume $\rho(0) = 0$,  $||x_v||_2 \leq B_x$ for all $v$,  $||\phi(x)||_{\infty} \leq b < \infty$ for all $x \in \mathbb{R}^r$, $\phi(0) = 0$, $g(0) = 0$. One possible choice of $\phi$ is a squashing function such as tanh. We also assume that $\phi$, $\rho$, and $g$ have Lipschitz constants $C_{\phi}$, $C_{\rho}$, and $C_g$ respectively; and that $W_1$ and $W_2$ have bounded norms: $||W_1||_2 \leq B_1$,  $||W_2||_2 \leq B_2$.  
The weights $W_1, W_2$ and functions $\phi,\rho,g$ are shared across nodes and layers. \\

\noindent The graph label is generated by a readout function that aggregates node embeddings of the final layer $L$. Here, we assume this function applies a local binary classifier of the form
$f_c(h^L_v) = \psi(\beta^\top h^L_v)$ from a family $\mathcal{F}_{\beta}$ parameterized by $\beta$ such that $||\beta||_2 \leq B_{\beta}$, with sigmoid function $\psi$ to each node representation $h^L_v$, and then averages the binary predictions of all nodes, i.e., $f(G) = \sum_{v \in V} f_c(h^L_v)$. We predict label $1$ if $f(G) > 0.5$, else $0$.
Such networks implement permutation invariance locally in each neighborhood, and globally when aggregating the node embeddings. This invariance will play an important role in the analysis.\\

\noindent Let $f(G)$ be the output of the entire GNN for input graph $G$ with true label $y \in \{0, 1\}$. Our loss is a margin loss applied to the difference in probability between true and incorrect label:
\begin{equation*}
  p(f(G),y) = y(2f(G) - 1) + (1-y)(1-2f(G)),
\end{equation*}
with $p(f(G),y) <0$ if and only if there is a classification error.
The margin loss is then, with $a = -p(f(G),y)$ and indicator function $\mathbf{1}[\cdot]$:
\begin{equation}
  \loss_\gamma(a) = \mathbf{1}[a > 0] + (1 + a/\gamma)\mathbf{1}[a \in [-\gamma,0]].
\end{equation}

\noindent A standard result in learning theory relates the population risk $\mathbb{P}[ p(f(G),y) \leq 0]$ to the empirical risk for training examples $\{(G_j,y_j)\}_{j=1}^m$
\begin{equation}
  \hat{\mathcal{R}}_{\gamma}(f) = \dfrac{1}{m} \sum_{j=1}^m \loss_{\gamma}(-p(f(G_j), y_j))
\end{equation}
and the empirical Rademacher complexity $\hat{\mathcal{R}}_S(\mathcal{J}_\gamma)$ of the class $\mathcal{J}_\gamma$ of functions concatenating the loss with the GNN prediction function $f$.
\begin{lemma}[\citet{MRT2012}] \label{MohriLemma}
  For any margin $\gamma > 0$, any prediction function $f$ in a class $\mathcal{F}$ and $\mathcal{J}_\gamma \in \{ (G,y) \mapsto \loss_\gamma(-p(f(G),y)) | f \in \mathcal{F}\}$, given $m$ samples $(G_j,y_j) \sim \mathcal{D}$, with probability $1-\delta$, the population risk for $\mathcal{D}$ and $f$ is bounded as
  \begin{equation*}
    \mathbb{P}(p(f(G), y) \leq 0) \leq \hat{\mathcal{R}}_{\gamma}(f) + 2 \hat{\mathcal{R}}_{\Gs}(\mathcal{J}_{\gamma}) +  3 \sqrt{\tfrac{\log(2/\delta)}{2m}}.
  \end{equation*}\label{lem:gbound}
\end{lemma}
\noindent Hence, we need to bound the empirical Rademacher complexity $\hat{\mathcal{R}}_{\Gs}(\mathcal{J}_{\gamma})$ for GNNs. We do this in two steps: (1) we show that it is sufficient to bound the Rademacher complexity of local node-wise computation trees; (2) we bound the complexity for a single tree via recursive spectral bounds, taking into account permutation invariance.


\subsection{From Graphs to Trees}
We begin by relating the Rademacher complexity of $\mathcal{J}_{\gamma}$ to the complexity of each node classification. The node embedding $h^L_v$ is equal to a function applied to the local computation tree of depth $L$, rooted at $v$, that we obtain when unrolling the $L$ neighborhood aggregations. That is, the tree represents the structured $L$-hop neighborhood of $v$, where the children of any node $u$ in the tree are the nodes in $\neigh(u)$. Hence, if $t$ is the tree at $v$, we may write, with a slight abuse of notation, $f_c(h^L_v) = f_c(t;\Theta)$, where $\Theta$ represents the parameters $W_1,W_2$ of the embedding and $\beta$ of the node classifier.\\

\noindent With this notation, we rewrite $f(G;\Theta)$ as an expectation over functions applied to trees. Let $T_1, \ldots, T_n$ be the set of all possible computation trees of depth $L$, and $w_i(G)$ the number of times $T_i$ occurs in $G$. Then, note that we may write  $f(G;\Theta)$ as the sum
\begin{equation}
   \sum_{i=1}^n \underbrace{\dfrac{w_i(G)}{\sum_{\ell=1}^n w_{\ell}(G)}}_{~=~ w'_i(G)} f_c(T_i;\Theta) = \mathbb{E}_{T \sim w'(G)} f_c(T;\Theta).
\end{equation}
This perspective implies a key insight of our analysis: the complexity of the GNN may be bounded by the complexity of the computation trees.
\begin{proposition} \label{Proposition4}
  Let $\mathcal{G} = \{G_1, \ldots, G_m\}$ be a set of i.i.d. graphs, and let $\mathcal{T} = \{t_1, \ldots, t_m\}$ be such that $t_j \sim w'(G_j), j \in [m]$. 
  Denote by $\hat{\mathcal{R}}_{\Gs}$ and $\hat{\mathcal{R}}_{\Ts}$ the empirical Rademacher complexity of GNNs for graphs $\mathcal{G}$ and trees $\Ts$, respectively. 
Then $\hat{\mathcal{R}}_{\mathcal{G}} \leq \mathbb{E}_{t_1, \ldots, t_m} \hat{\mathcal{R}}_{\Ts}.$
\end{proposition}

\noindent Therefore, to apply Lemma~\ref{lem:gbound}, it is sufficient to bound the Rademacher complexity of classifying single node-wise computation trees. Before addressing this next step in detail, we state and discuss our main result for this section.

\subsection{Generalization Bound for GNNs}
We define the {\em percolation complexity} of our GNNs to be $\mathcal{C} \triangleq C_{\rho}C_{g}C_{\phi}B_{2}$. 
We now bound  $\hat{\mathcal{R}}_{\mathcal{T}}(\mathcal{J}_\gamma)$, when each tree $t_j \in \mathcal{T}$ has a  branching factor (i.e., maximum number of neighbors for any node) at most $d$, and $\mathcal{J}_\gamma$ maps each $(t, y)$ pair to  $\loss_\gamma(-p(f_c(t; \Theta),y))$.   

\begin{proposition} \label{PropGen}
The empirical Rademacher complexity of $\mathcal{J}_{\gamma}$ with respect to $\mathcal{T}$ is
\begin{eqnarray*}
\hat{\mathcal{R}}_{\mathcal{T}}(\mathcal{J}_{\gamma}) \leq \dfrac{4}{\gamma m} + \dfrac{24 r B_{\beta} Z}{\gamma \sqrt{m}} \sqrt{3\log Q}~, ~~ \text{where}
\end{eqnarray*}
$$Q = 24B_{\beta} \sqrt{m} \max\{Z, M \sqrt{r} \max\{B_x B_{1}, \overline{R} B_{2}\}\},$$
$$M = C_{\phi} \dfrac{\left(\mathcal{C} d\right)^L - 1}{\mathcal{C} d - 1}~,~~~ Z = C_{\phi}B_{1}B_x + C_{\phi}B_{2}\overline{R}~,$$
$$\overline{R} \leq C_\rho C_g d \min\left\{b\sqrt{r}, B_{1}B_x M \right\}~. $$
\normalsize

\end{proposition}

\noindent Note that we do not need to prespecify $B_{1}$ and $B_{2}$: we can simply take these values to be the spectral norm, respectively, of the learned weights $W_1$ and $W_2$. 
Before proceeding with the proof, we discuss some important implications of this result in the wake of Lemma \ref{MohriLemma} and Proposition \ref{Proposition4}.

\paragraph{Comparison with RNN.}
We summarize below the dependence of our generalization error on the various parameters for different $\mathcal{C}$ up to log factors (denoted by notation $\tilde{\mathcal{O}}$).  We also mention the corresponding bounds for RNN on a sequence of length $L$ when the spectral norm of recurrent weights in RNN is respectively less than, equal to, or greater than 1 (note that we renamed some parameters from \citet{CLZ2019} for notational consistency).
\begin{table}[h!]
  \begin{center}
    \label{tab:table1}
    \begin{tabular}{l|c|r} 
      $\mathcal{C}$ &  {\bf GNN (ours)}  & {\bf RNN (\citet{CLZ2019})}   \\
      \hline
      $< 1/d$ & $\tilde{\mathcal{O}}\left(\dfrac{rd}{\sqrt{m}\gamma}\right)$ & $\tilde{\mathcal{O}}\left(\dfrac{r}{\sqrt{m}\gamma}\right)$\\
      $= 1/d$ & $\tilde{\mathcal{O}}\left(\dfrac{rdL}{\sqrt{m}\gamma}\right)$ & $\tilde{\mathcal{O}}\left(\dfrac{rL}{\sqrt{m}\gamma}\right)$\\
      $> 1/d$ & $\tilde{\mathcal{O}}\left(\dfrac{rd\sqrt{rL}}{\sqrt{m}\gamma}\right)$ & $\tilde{\mathcal{O}}\left(\dfrac{r\sqrt{rL}}{\sqrt{m}\gamma}\right)$
    \end{tabular}
  \end{center}
\end{table}
Our analysis implies GNNs have essentially the same dependence on dimension $r$, depth $L$, and sample size $m$ as RNN. The additional dependence on branching factor $d$ is due to processing trees, in contrast to processing sequences in RNNs.
\paragraph{Comparison with VC-bounds for GNNs.}
\rm \citet{STH2018}  proved that the dependence of VC-dimension in their setting for tanh and logistic sigmoid activations is fourth order with respect to the number of hidden units $H$, and quadratic in each of $r$ and maximum number of nodes $N$ in any input graph. Note that $N$ is at least $d$, and possibly much larger than $d$.  Since $H=r$ in our setting, this amounts to having VC-dimension scale as $\mathcal{O}(r^6N^2)$, and consequently, generalization error scale as $\tilde{\mathcal{O}}(r^3N/\sqrt{m})$. Thus, our generalization bounds are significantly tighter even when $L = \mathcal{O}(r)$.

\paragraph{Role of local permutation invariance.}
Note that previous works were confined to dealing with permutation-invariance at a global level (\citet{SI2019, SGSR2017}). 
In contrast, we perform a composition of permutation-invariant transformations, where each transformation pertains to applying a permutation at a level of the input tree. 
We exploit local permutation-invariance via sum-decomposability. 
In the absence of local permutation-invariance, we would need to deal with the orderings at each node separately.

\paragraph{Extension to other variants.} Note that we define $C_g$ and $C_{\rho}$ with respect to the aggregate function (e.g., unweighted sum or mean) that acts prior to transformation by $W_2$.
We can easily extend our analysis to include the edge embeddings.
We considered message passing in the context of so-called {\em mean field embedding} \citep{DDS2016}, however, other updates, such as {\em embedded loopy belief propagation}, may be analyzed similarly in our framework. 


\paragraph{Analysis.} Our proof follows a multistep analysis. We first quantify the effect of change in shared variables on the embedding of the root node of a fixed tree. 
This step entails a recursive decomposition over the effect on subtrees. Specifically, we recursively bound the effect on each subtree of the root by the maximum effect across these subtrees. Since both the non-linear activation function and the permutation-invariant aggregation function are Lipschitz-continuous, and the feature vector at the root and the shared weights have bounded norm, the embedding at the root of the tree adapts to the embeddings from the subtrees. We then quantify the effect of changing not only the shared weights but also changing the classifier parameters. Since the classifier parameters are chosen from a bounded norm family, we can bound the change in prediction probability.
This allows us to use a covering number argument to approximate the predictions, and 
subsequently, bound the empirical Rademacher complexity via Dudley's entropy integral.\\

\noindent Fix the feature vectors for the computation tree of depth $L$ having degree of each internal node equal to $d$. Let the feature vector associated with the root (assumed to be at level $L$) of the tree be given by $x_{L}$.  We denote the feature vector associated with node $j$ at level $\ell \in [L-1] \triangleq \{1, 2, \ldots, L-1\}$ by $x_{\ell, j}$. Denote the embedding produced by the subtree rooted at node $j$ on level $\ell \in [L-1]$ by $T_{\ell, j}(W_a, W_b)$ when $W_a$ and $W_b$ are the parameters of the model. Consider two sets of parameters $\{W_1, W_2\}$ and $\{W_1', W_2'\}$. We will denote the embedding  vector produced by the GNN after processing the entire tree by $T_{L}(W_1, W_2)$ as a shorthand for $T_{L, 1}(W_1, W_2)$. Denote the set of subtrees of node with feature vector $x$ by $C(x)$. We structure the proof as a sequence of several sub-results for improved readability. We first quantify the change in embedding due to change in the shared weight parameters.   

\begin{lemma}\label{Lemma1}
The $l_2$-norm of difference of embedding vectors produced by $(W_1, W_2)$ and $(W_1', W_2')$ 
after they process the tree all the way from the leaf level to the root can be bounded recursively as 
\begin{eqnarray*}
\Delta_L & \triangleq &   \left|\right|T_L(W_1, W_2) - T_L(W_1', W_2')\left|\right|_2\\
& \leq & C_{\phi}B_x\left|\left|(W_1-W_1')\right|\right|_2 \nonumber ~+~ \mathcal{C} d \max_{j \in C(x_{L})} \Delta_{L-1, j} \label{eq4} \\ &+& C_{\phi}||(W_2-W_2') R(W_1, W_2, x_L)||_2 \nonumber~,
\end{eqnarray*}
where
\begin{eqnarray*}
R(W_1, W_2, x_L) = \rho\big(\sum_{j \in C(x_{L})} g(T_{L-1,j} (W_1, W_2))\big)
\end{eqnarray*}
is the permutation-invariant aggregation of the embeddings of the subtrees rooted at level $L-1$ under $(W_1, W_2)$.   \end{lemma}

\noindent We therefore proceed to bounding $||R(W_1, W_2, x_L)||_2$.  
\begin{lemma} \label{Lemma2}
\begin{eqnarray*}
& ||R(W_1, W_2, x_L)||_2\\ 
 \leq &C_\rho C_g d \min\left\{b\sqrt{r}, C_{\phi}B_{1}B_x  \dfrac{(\mathcal{C} d)^{L}-1}{\mathcal{C} d-1}\right\}
\end{eqnarray*}
\end{lemma}
\noindent We now quantify the change in probability (that the tree label is 1) $\Lambda_L$ due to change in both the shared weights and the classifier parameters. We prove the following result.
\begin{lemma} \label{Lemma3}
The change in probability $\Lambda_L$ due to change in   parameters from $(W_1, W_2, \beta)$ to $(W_1', W_2', \beta')$ is  
\begin{eqnarray*}
\Lambda_L& = & |\psi(\beta^{\top} T_L(W_1, W_2)) - \psi( {\beta'}^{\top} T_L(W_1', W_2'))| \\
& \leq & ||\beta-\beta'||_2 Z + B_{\beta} \Delta_L~, 
\end{eqnarray*}
where $Z$ is an upper bound on  $||T_L(W_1, W_2)||_2$. Moreover, we can bound $\Delta_L$ non-recursively: 
$$\Delta_L ~~\leq~~  M B_x\left|\left|W_1-W_1'\right|\right|_2$$
$$ ~~+~~ M  ||R(W_1, W_2, x_L)||_2 ||W_2-W_2'||_2~.$$
\end{lemma}

\noindent Lemma \ref{Lemma3} allows us to ensure that $\Lambda_L$ is small via a sufficiently large covering. Specifically, we show the following.

\begin{lemma} \label{Lemma4}
The change in probability $\Lambda_L$ can be bounded by $\epsilon$ using a covering of size $P$, where $\log P$ is at most 
\begin{eqnarray*}
3r^2\log \left(1 + \dfrac{6B_{\beta} \max\{Z, M \sqrt{r} \max\{B_x B_{1}, \overline{R} B_{2}\}\}}{\epsilon} \right)~. 
\end{eqnarray*}
Moreover, when $$\epsilon < 6B_{\beta} \max\{Z, M \sqrt{r} \max\{B_x B_{1}, \overline{R} B_{2}\}\},$$
a covering of size $P$ such that
$\log P$ is at most
$$3r^2 \log\left(\dfrac{12B_{\beta} \max\{Z, M \sqrt{r} \max\{B_x B_{1}, \overline{R} B_{2}\}\}}{\epsilon}\right)$$
suffices to ensure $\Lambda_L \leq \epsilon$. Here,
$\overline{R} \triangleq ||R(W_1, W_2, x_L)||_2$, and $Z$, $M$ are as defined in the statement of Proposition  \ref{PropGen}.
\end{lemma}

\noindent The remaining steps for Proposition \ref{PropGen} are straightforward and deferred to the Supplementary. We now outline an approach to facilitate future work on understanding the generalization ability of CPNGNNs. 

\subsection{Toward generalization analysis for CPNGNNs}
Two parts were integral to our analysis: {\bf (a)} bounding complexity via local computation trees, and
{\bf (b)} the sum decomposition property of permutation-invariant functions. We now provide their counterparts for CPNGNNs.\\   

\noindent Like before, we start with a vertex $v$, and unroll the $L$ neighborhood aggregations to obtain a computation tree of depth $L$, rooted at $v$. However, now we additionally label each edge in the computation tree with the respective ports of the nodes incident on the edge (enabled by consistent ordering). Thus, we may analyze a input port-numbered graph using its node-wise port-numbered trees. \\    

 \noindent Again, note that permutation-invariance  applies to multisets of messages but
 not port-numbered messages, therefore in CPNGNNs, we cannot express the aggregation as in \eqref{SumDecomp}. 
 To address this issue, we 
now provide an injective function for aggregating a collection of port-numbered messages. The function takes a general sum-form that interestingly decouples the dependence on each message and its corresponding port number. 
\begin{proposition} \label{Injective}
Assume $\mathcal{X}$ is countable. There exists a function $f : \mathcal{X} \times \mathcal{P} \mapsto \mathcal{R}^{n}$ such that\\ \noindent $h((x_1, p_1), \ldots, (x_{|P|}, p_{|P|})) = \sum_{i \in [|P|]} g(p_i) f(x_i)$ for each port-numbered sequence of $(x_i, p_i)$ pairs, where $P \subset \mathbb{N}$, $X = \{x_1, x_2, \ldots, x_{|P|}\} \subset \mathcal{X}$ is a multiset of bounded size, and $p_i$ are all distinct numbers from $[|P|]$.  
\end{proposition}
\noindent The result in Proposition \ref{Injective} holds particular significance, since it is known \citep{Hella2015} that port-numbered messages provide a strictly richer class than sets and multisets. 
The generalization bound for CPNGNN will be worse than the result in Proposition \ref{PropGen}, since each port appears as an exponent in our generalized decomposition and thus the complexity of aggregation grows rapidly in the neighborhood size.
We leave a detailed analysis of the generalization ability of CPNGNNs for future work.

\bibliographystyle{plainnat}
\bibliography{references.bib}
\appendix
\section{Supplementary material}
We now provide detailed proofs for all our propositions and lemmas.\\ \\

\noindent \textbf{Proof of Proposition \ref{WLvPN}} 
\begin{proof}
We show that CPNGNN, using some consistent port ordering, can distinguish some non-isomorphic graphs that LU-GNNs cannot.  
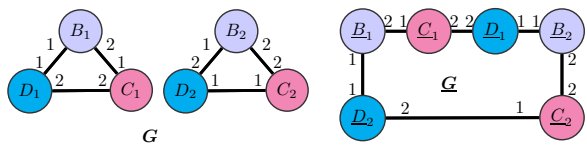
\begin{figure}[h]
\centering
\resizebox{0.47\textwidth}{!}{\begin{tikzpicture}
\tikzstyle{main}=[circle, minimum size = 5mm, thick, draw =black!80, node distance = 6mm]
\tikzstyle{connect}=[-latex, thick]
  \node[main, fill=blue!20] (PB1) {$B_1$};
  \node[main, fill=cyan] (PC1) [below=of PB1, xshift=-1cm, yshift=0.3cm] {$D_1$}; 
  \node[main, fill=magenta!60] (PC2) [below=of PB1, xshift=1cm, yshift=0.3cm] {$C_1$}; 
  \node[main, fill=blue!20] (PB2) [right=of PB1, xshift=1.6cm] {$B_2$};
  \node[main, fill=cyan] (PC3) [below=of PB2, xshift=-1cm, yshift=0.3cm] {$D_2$}; 
  \node[main, fill=magenta!60] (PC4) [below=of PB2, xshift=1cm, yshift=0.3cm] {$C_2$}; 
  

 \draw [ultra thick] (PB1) -- (PC1) node [near start, yshift=0.2cm, xshift=-0.2cm] {1};
 \draw [ultra thick] (PB1) -- (PC1) node [near end, yshift=0.1cm, xshift=-0.2cm] {1};
 \draw [ultra thick] (PB1) -- (PC2) node [near start, yshift=0.2cm, xshift=0.2cm] {2};
 \draw [ultra thick] (PB1) -- (PC2) node [near end, yshift=0.1cm, xshift=0.2cm] {1};
 \draw [ultra thick] (PC1) -- (PC2) node [very near start, yshift=0.2cm] {2};
 \draw [ultra thick] (PC1) -- (PC2) node [very near end, yshift=0.2cm] {2};
 
 \draw [ultra thick] (PB2) -- (PC3) node [near start, yshift=0.2cm, xshift=-0.2cm] {1};
 \draw [ultra thick] (PB2) -- (PC3) node [near end, yshift=0.1cm, xshift=-0.2cm] {2};
 \draw [ultra thick] (PB2) -- (PC4) node [near start, yshift=0.2cm, xshift=0.2cm] {2};
 \draw [ultra thick] (PB2) -- (PC4) node [near end, yshift=0.1cm, xshift=0.2cm] {2};
 \draw [ultra thick] (PC3) -- (PC4) node [very near start, yshift=0.2cm] {1};
 \draw [ultra thick] (PC3) -- (PC4) node [very near end, yshift=0.2cm] {1};
 
 \node[main, fill=blue!20] (PBa) [right=of PB2,  xshift=1cm] {$\underline{B}_1$};
 \node[main, fill=magenta!60] (PCa) [right=of PBa, xshift=-0.2cm] {$\underline{C}_1$};
 \node[main, fill=cyan] (PCb) [right=of PCa, xshift=-0.2cm] {$\underline{D}_1$};
 \node[main, fill=blue!20] (PBb) [right=of PCb, xshift=-0.2cm] {$\underline{B}_2$};
 \node[main, fill=magenta!60] (PCd) [below=of PBb, yshift=-0.2cm] {$\underline{C}_2$};
 \node[main, fill=cyan] (PCc) [below=of PBa, yshift=-0.2cm] {$\underline{D}_2$};
 
 \node[rectangle,draw=white!100, below=of PCc,yshift=2.4cm, xshift=1.7cm] {\bm{$\underline{G}$}};
 
 \draw [ultra thick] (PBa) -- (PCa) node [very near start, yshift=0.2cm] {2};
 \draw [ultra thick] (PBa) -- (PCa) node [very near end, yshift=0.2cm] {1};
 \draw [ultra thick] (PCa) -- (PCb) node [very near start, yshift=0.2cm] {2};
 \draw [ultra thick] (PCa) -- (PCb) node [very near end, yshift=0.2cm] {2};
 \draw [ultra thick] (PCb) -- (PBb) node [very near start, yshift=0.2cm] {1};
 \draw [ultra thick] (PCb) -- (PBb) node [very near end, yshift=0.2cm] {1};
 \draw [ultra thick] (PBb) -- (PCd) node [very near start, xshift=0.2cm] {2};
 \draw [ultra thick] (PBb) -- (PCd) node [very near end, xshift=0.2cm] {2};
 \draw [ultra thick] (PCc) -- (PCd) node [very near start, yshift=0.2cm] {2};
 \draw [ultra thick] (PCc) -- (PCd) node [very near end, yshift=0.2cm] {1};
 \draw [ultra thick] (PCc) -- (PBa) node [very near start, xshift=-0.2cm] {1};
 \draw [ultra thick] (PCc) -- (PBa) node [very near end, xshift=-0.2cm] {1};
\node[rectangle,draw=white!100, below=of PC3, yshift=0.8cm, xshift=-0.7cm] {\bm{$G$}};
\end{tikzpicture}}
\end{figure}

\noindent We construct a pair of graphs \boldmath{$G$} and  \boldmath{$\underline{G}$} such that \boldmath{$G$} consists of two triangles that differ in port-ordering but are otherwise identical, while \boldmath{$\underline{G}$} (indicated by underlined symbols) consists of a single even-length cycle. The construction ensures that each node labeled with $X \in \{B_1, C_1, D_1, B_2, C_2\}$ in \boldmath{$G$} has the same identical view (i.e., indistinguishable node features, and neighborhood) as the corresponding node labeled $\underline{X}$ in \boldmath{$\underline{G}$}. However, $D_2$ and $\underline{D_2}$ have distinguishable neighborhoods due to different port-numbers: e.g., $D_2$ is connected to $B_2$ at port 2, whereas $\underline{D}_2$ is connected to $\underline{B}_1$ at port 1. Likewise, $D_2$ is connected to $C_2$ at port 1, in contrast to $\underline{D}_2$ that is connected to $\underline{C}_2$ at port 2. However, LU-GNN does not incorporate any spatial information such as ports, and fails to tell one graph from the other.        
\end{proof}

\noindent Note that since $\angle B_1C_1D_1$ differs from $\angle \underline{B}_1\underline{C}_1\underline{D}_1$, DimeNet can also distinguish between the two graphs. \\

\noindent \textbf{Proof of Proposition \ref{PortIssues}} 
\begin{proof}
We now illustrate the importance of choosing a good consistent port numbering. Specifically, we construct a pair of graphs, and two different consistent port numberings $p$ and $q$ such that CPNGNN can distinguish the graphs with $p$ but not $q$.

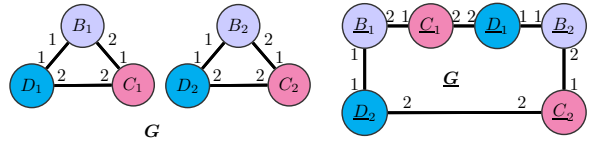
\begin{figure}[h]
\centering
\resizebox{0.47\textwidth}{!}{\begin{tikzpicture}
\tikzstyle{main}=[circle, minimum size = 5mm, thick, draw =black!80, node distance = 6mm]
\tikzstyle{connect}=[-latex, thick]
  \node[main, fill=blue!20] (PB1) {$B_1$};
  \node[main, fill=cyan] (PC1) [below=of PB1, xshift=-1cm, yshift=0.3cm] {$D_1$}; 
  \node[main, fill=magenta!60] (PC2) [below=of PB1, xshift=1cm, yshift=0.3cm] {$C_1$}; 
  
  \node[main, fill=blue!20] (PB2) [right=of PB1, xshift=1.6cm] {$B_2$};
  \node[main, fill=cyan] (PC3) [below=of PB2, xshift=-1cm, yshift=0.3cm] {$D_2$}; 
  \node[main, fill=magenta!60] (PC4) [below=of PB2, xshift=1cm, yshift=0.3cm] {$C_2$}; 
  

 \draw [ultra thick] (PB1) -- (PC1) node [near start, yshift=0.2cm, xshift=-0.2cm] {1};
 \draw [ultra thick] (PB1) -- (PC1) node [near end, yshift=0.1cm, xshift=-0.2cm] {1};
 \draw [ultra thick] (PB1) -- (PC2) node [near start, yshift=0.2cm, xshift=0.2cm] {2};
 \draw [ultra thick] (PB1) -- (PC2) node [near end, yshift=0.1cm, xshift=0.2cm] {1};
 \draw [ultra thick] (PC1) -- (PC2) node [very near start, yshift=0.2cm] {2};
 \draw [ultra thick] (PC1) -- (PC2) node [very near end, yshift=0.2cm] {2};
 
 \draw [ultra thick] (PB2) -- (PC3) node [near start, yshift=0.2cm, xshift=-0.2cm] {1};
 \draw [ultra thick] (PB2) -- (PC3) node [near end, yshift=0.1cm, xshift=-0.2cm] {1};
 \draw [ultra thick] (PB2) -- (PC4) node [near start, yshift=0.2cm, xshift=0.2cm] {2};
 \draw [ultra thick] (PB2) -- (PC4) node [near end, yshift=0.1cm, xshift=0.2cm] {1};
 \draw [ultra thick] (PC3) -- (PC4) node [very near start, yshift=0.2cm] {2};
 \draw [ultra thick] (PC3) -- (PC4) node [very near end, yshift=0.2cm] {2};
 
 
 \node[main, fill=blue!20] (PBa) [right=of PB2,  xshift=1cm] {$\underline{B}_1$};
 \node[main, fill=magenta!60] (PCa) [right=of PBa, xshift=-0.2cm] {$\underline{C}_1$};
 \node[main, fill=cyan] (PCb) [right=of PCa, xshift=-0.2cm] {$\underline{D}_1$};
 \node[main, fill=blue!20] (PBb) [right=of PCb, xshift=-0.2cm] {$\underline{B}_2$};
 \node[main, fill=magenta!60] (PCd) [below=of PBb, yshift=-0.2cm] {$\underline{C}_2$};
 \node[main, fill=cyan] (PCc) [below=of PBa, yshift=-0.2cm] {$\underline{D}_2$};
 
 \node[rectangle,draw=white!100, below=of PCc,yshift=2.4cm, xshift=1.7cm] {\bm{$\underline{G}$}};
 
 \draw [ultra thick] (PBa) -- (PCa) node [very near start, yshift=0.2cm] {2};
 \draw [ultra thick] (PBa) -- (PCa) node [very near end, yshift=0.2cm] {1};
 \draw [ultra thick] (PCa) -- (PCb) node [very near start, yshift=0.2cm] {2};
 \draw [ultra thick] (PCa) -- (PCb) node [very near end, yshift=0.2cm] {2};
 \draw [ultra thick] (PCb) -- (PBb) node [very near start, yshift=0.2cm] {1};
 \draw [ultra thick] (PCb) -- (PBb) node [very near end, yshift=0.2cm] {1};
 \draw [ultra thick] (PBb) -- (PCd) node [very near start, xshift=0.2cm] {2};
 \draw [ultra thick] (PBb) -- (PCd) node [very near end, xshift=0.2cm] {1};
 \draw [ultra thick] (PCc) -- (PCd) node [very near start, yshift=0.2cm] {2};
 \draw [ultra thick] (PCc) -- (PCd) node [very near end, yshift=0.2cm] {2};
 \draw [ultra thick] (PCc) -- (PBa) node [very near start, xshift=-0.2cm] {1};
 \draw [ultra thick] (PCc) -- (PBa) node [very near end, xshift=-0.2cm] {1};
\node[rectangle,draw=white!100, below=of PC3, yshift=0.8cm, xshift=-0.7cm] {\bm{$G$}};
\end{tikzpicture}}
\end{figure}
\noindent We modify the consistent port numbering from the construction of Proposition \ref{WLvPN}. We consider the same pair of graphs as in the proof of Proposition  \ref{WLvPN}. However, instead of having different numberings for the two components (i.e., triangles) of {\boldmath $G$}, we now carry over the ordering from one component to the other. The two components become identical with this modification. For any node labeled $\underline{X}_1$ or $\underline{X}_2$, and any neighbor labeled $\underline{Y}_1$ or $\underline{Y}_2$, $X, Y \in \{B, C, D\}$,  we can now simply assign the same respective local ports as the nodes labeled $X_1$ and $Y_1$ (or, equivalently, $X_2$ and $Y_2$). It is easy to verify that the two graphs become port-locally isomorphic under the new ordering, and thus cannot be separated with any permutation-invariant readout (using Proposition \ref{Bijection}).   \\       
\end{proof}



\noindent \textbf{Proof of Proposition \ref{Bijection}} 
\begin{proof}
If the surjection $f: V_1 \to V_2$ in Definition \ref{Def1} is also injective, then we can simply take $h=f$. Therefore, we focus on the case when $f$ is not injective.  We will show that $f$ can be used to inform $h$.  Since $f$ is not injective, there exist $v_1, v_1' \in V_1$ such that $v_1 \neq v_1'$ but $f(v_1) = f(v_1') = v_2$ for some $v_2 \in V_2$. Then, by condition ${\textbf (a)}$ in Definition \ref{Def1}, we immediately get that the feature vector
\begin{equation} \label{eqBijection} x_{v_1} = x_{f(v_1)} = x_{f(v_1')} = x_{v_1'}~. \end{equation}
Moreover, by other conditions, there is a consistent port bijection from neighborhood of $v_1$ to that of $v_2$, and likewise another bijection from neighborhood of $v_1'$ to that of $v_2$. Therefore, there is a consistent port bijection from neighborhood of $v_1$ to that of $v_1'$. Together with \eqref{eqBijection} and our assumption that $f(v_1) = f(v_1') = v_2$, this implies that $v_1$ and $v_1'$ are locally {\em indistinguishable}. Note that there could be more such nodes that are indistinguishable from $v_1$ (or $v_1'$), e.g., when all such nodes map to $v_2$ as well. \\

\noindent Without loss of generality, let $\mathcal{E}_1(v_1) \subseteq V_1$ denote the equivalence class of all nodes, including $v_1$, that are indistinguishable from $v_1$ in graph $G_1$. Similarly, let $\mathcal{E}_2(v_2) \subseteq V_2$ be the class of nodes indistinguishable from $v_2$ in $G_2$. Consider $\ell_1 = |\mathcal{E}_1(v_1)|$ and $\ell_2 = |\mathcal{E}_2(v_2)|$. 
We claim that $\ell_1 = \ell_2$. Suppose not. Then if $\ell_1 < \ell_2$, we can have $h$ map each node in $\mathcal{E}_1(v_1)$ to a separate node in $\mathcal{E}_2(v_2)$, and use the same mapping as $f$ on the other nodes in $V_1$. Doing so does not decrease the co-domain of $V_2$, and $h$ remains surjective. We are therefore left with $\ell_2 - \ell_1 > 0$ nodes from $\mathcal{E}_2(v_2)$. Therefore, these nodes must have at least one preimage in the set $V_1 - \mathcal{E}_1(v_1)$ since $f$ (and thus $h$) is a surjection by assumption ${\bf (a)}$ in Definition 1. 
This is clearly a contradiction since any such preimage must have either a different feature vector, or a non-isomorphic port-consistent neighborhood. 
By a symmetric argument, using the surjection of map from $V_2$ to $V_1$, we conclude that $\ell_1 = \ell_2$. Note that $h$ did not tinker with the nodes that were outside the class $\mathcal{E}_1(v_1)$.
Recycling the procedure for other nodes in $V_1 - \mathcal{E}_1(v_1)$ that might map under $f$ to a common image in $V_2$, we note that $h$ ends up being injective. Since $h$ remains surjective throughout the procedure, we conclude that $h$ is a bijection.  \\

\noindent We now prove by induction that
the corresponding nodes in port-locally isomorphic graphs
have identical embeddings
for any CPNGNN. Consider any such GNN with $L+1$ layers parameterized by the sequence $\theta_{1:L+1} \triangleq (\theta_1, \ldots, \theta_L, \theta_{L+1})$. Since there exists a bijection $h$ such that any node $v_1 \in G_1$ has an identical local view (i.e., node features, and port-numbered neighbors) as $v_2 = h(v_1) \in G_1$, the 
updated embeddings for $v_1$ and $v_2$ are identical after the first layer. Assume that these embeddings remain identical after update from each layer $\ell \in \{2, 3, \ldots, L\}$. Since $v_1$ and $v_2$ have identical local views and have identical embedding from the $L${\em th} layer, the updates for these nodes by the $(L+1)${\em th} layer are identical. Therefore, $v_1$ and $v_2$ have identical embeddings. Since $h$ is a bijection, for every $v \in V_1$ there is a corresponding $h(v) \in V_2$ with the same embedding, and thus both $G_1$ and $G_2$ produce the same output with any permutation readout function. Our choice of $\theta_{1:L+1}$ was arbitrary, so the result follows. \\
\end{proof}


\noindent \textbf{Proof of Proposition \ref{prop1}} 
\begin{proof}
We now show that there exist consistent port orderings such that CPNGNNs with permutation-invariant readout cannot decide several important graph properties: girth, circumference, diameter, radius, conjoint cycle, total number of cycles, and $k$-clique. The same result also holds for LU-GNNs where nodes do not have access to any consistent port numbering. \\

 \noindent We first construct a pair of graphs that have cycles of different length but produce the same output embedding via the readout function. Specifically, we show that CPNGNNs cannot decide a graph having cycles of length $n$ from a cycle of length $2n$. We construct a counterexample for $n=4$. Our first graph consists of two cycles of length 4 (each denoted by $S_4$), while the other graph is a cycle of length 8 (denoted by $S_8$). We associate identical feature vectors with nodes that have the same color, or equivalently, that are marked with the same symbol ignoring the subscripts and the underline. For example, $A_1$, $A_2$, $\underline{A}_1$, and $\underline{A}_2$ are all assigned the same feature vector. Moreover, we assign identical edge feature vectors to edges that have the same pair of symbols at the nodes. 

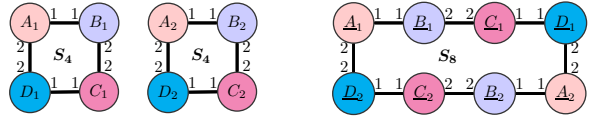
\begin{figure}[h]
    \centering
    \resizebox{0.47\textwidth}{!}{\begin{tikzpicture}
\tikzstyle{main}=[circle, minimum size = 5mm, thick, draw =black!80, node distance = 6mm]
\tikzstyle{connect}=[-latex, thick]
  \node[main, fill=red!20] [xshift=2cm] (PA) {$A_1$};
  \node[main, fill=blue!20] (PB) [right=of PA] {$B_1$};
  \node[main, fill=magenta!60] (PC) [below=of PB] {$C_1$};
 \node[main, fill=cyan] (PD) [below=of PA] {$D_1$};
 
 
 \node[main, fill=red!20] [right=of PB] (PA2) {$A_2$};
  \node[main, fill=blue!20] (PB2) [right=of PA2] {$B_2$};
  \node[main, fill=magenta!60] (PC2) [below=of PB2] {$C_2$};
 \node[main, fill=cyan] (PD2) [below=of PA2] {$D_2$};

 \node[main, fill=red!20] (L1) [right=of PB2, xshift=1cm] {$\underline{A}_1$};
  \node[main, fill=blue!20] (L2) [right=of L1] {$\underline{B}_1$};
  \node[main, fill=magenta!60] (L3) [right=of L2] {$\underline{C}_1$};
  \node[main, fill=cyan] (Lt) [right=of L3] {$\underline{D}_1$};
  \node[main,fill=cyan] (O1) [below=of L1] {$\underline{D}_2$};
  \node[main,fill=magenta!60] (O2) [right=of O1,below=of L2] {$\underline{C}_2$};
  \node[main,fill=blue!20] (O3) [right=of O2,below=of L3] {$\underline{B}_2$};
  \node[main,fill=red!20] (Ot) [right=of O3,below=of Lt] {$\underline{A}_2$};

\draw [ultra thick] (PA) -- (PB) node [very near start, yshift=0.2cm] {1};
 \draw [ultra thick] (PA) -- (PB) node [very near end, yshift=0.2cm] {1}; 
 
 \draw [ultra thick] (PB) -- (PC) node [very near start, xshift=0.2cm] {2};
 \draw [ultra thick] (PB) -- (PC) node [very near end, xshift=0.2cm, yshift=0.05cm] {2}; 
 \draw [ultra thick] (PC) -- (PD) node [very near start, yshift=0.2cm] {1};
 \draw [ultra thick] (PC) -- (PD) node [very near end, yshift=0.2cm] {1}; 
 
 \draw [ultra thick] (PD) -- (PA) node [very near start, xshift=-0.2cm, yshift=1pt] {2};
 \draw [ultra thick] (PD) -- (PA) node [very near end, xshift=-0.2cm] {2};
 
 \draw [ultra thick] (PA2) -- (PB2) node [very near start, yshift=0.2cm] {1};
 \draw [ultra thick] (PA2) -- (PB2) node [very near end, yshift=0.2cm] {1}; 
 
 \draw [ultra thick] (PB2) -- (PC2) node [very near start, xshift=0.2cm] {2};
 \draw [ultra thick] (PB2) -- (PC2) node [very near end, xshift=0.2cm, yshift=0.05cm] {2}; 
 \draw [ultra thick] (PC2) -- (PD2) node [very near start, yshift=0.2cm] {1};
 \draw [ultra thick] (PC2) -- (PD2) node [very near end, yshift=0.2cm] {1}; 
 
 \draw [ultra thick] (PD2) -- (PA2) node [very near start, xshift=-0.2cm, yshift=1pt] {2};
 \draw [ultra thick] (PD2) -- (PA2) node [very near end, xshift=-0.2cm] {2};
 
 \draw [ultra thick] (L1) -- (L2) node [very near start, yshift=0.2cm] {1};
 \draw [ultra thick] (L1) -- (L2) node [very near end, yshift=0.2cm] {1};
 
 \draw [ultra thick] (L2) -- (L3) node [very near start, yshift=0.2cm] {2};
 \draw [ultra thick] (L2) -- (L3) node [very near end, yshift=0.2cm] {2};
 
 \draw [ultra thick] (L3) -- (Lt) node [very near start, yshift=0.2cm] {1};
 \draw [ultra thick] (L3) -- (Lt) node [very near end, yshift=0.2cm] {1};
 
 \draw [ultra thick] (O1) -- (O2) node [very near start, yshift=0.2cm] {1};
 \draw [ultra thick] (O1) -- (O2) node [very near end, yshift=0.2cm] {1};
 
 \draw [ultra thick] (O2) -- (O3) node [very near start, yshift=0.2cm] {2};
 \draw [ultra thick] (O2) -- (O3) node [very near end, yshift=0.2cm] {2};
 
 \draw [ultra thick] (O3) -- (Ot) node [very near start, yshift=0.2cm] {1};
 \draw [ultra thick] (O3) -- (Ot) node [very near end, yshift=0.2cm] {1};
 
 \draw [ultra thick] (O1) -- (L1) node [very near start, xshift=-0.2cm, yshift=1pt] {2};
 \draw [ultra thick] (O1) -- (L1) node [very near end, xshift=-0.2cm] {2};
 
 \draw [ultra thick] (Ot) -- (Lt) node [very near start, xshift=0.2cm, yshift=1pt] {2};
 \draw [ultra thick] (Ot) -- (Lt) node [very near end, xshift=0.2cm] {2};

        \foreach \from/\to in {L1/O1, L1/L2, L2/L3, L3/Lt, Lt/Ot, O1/O2, O2/O3, O3/Ot, PA/PB, PB/PC, PC/PD, PD/PA,
        PA2/PB2, PB2/PC2, PC2/PD2, PD2/PA2}
 \draw [ultra thick] (\from) -- (\to); 
         \node[rectangle,draw=white!100, below=of PA, yshift=1cm, xshift=0.7cm] {\bm{$S_4$}};
         \node[rectangle,draw=white!100, below=of PA2, yshift=1cm, xshift=0.7cm] {\bm {$S_4$}};
\node[rectangle,draw=white!100, below=of L1, yshift=1cm, xshift=2cm] {\bm{$S_8$}};
    
\end{tikzpicture}}
\end{figure}

\noindent Thus, we note that a bijection exists between the two graphs with node $X$ in the first graph corresponding to $\underline{X}$ in the second graph such that both the nodes have identical features and indistinguishable port-ordered neighborhoods. Since, the two graphs have different girth, circumference, diameter, radius, and total number of cycles, it follows from Proposition \ref{Bijection} that CPNGNN cannot decide these properties. Note that the graph with two $S_4$ cycles is disconnected, and hence its radius (and diameter) is $\infty$. 
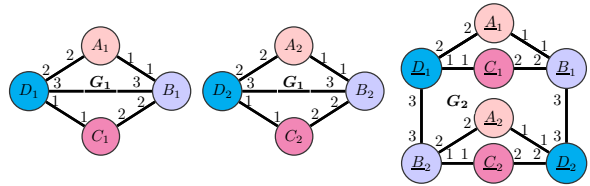
\begin{figure}[h]
\centering
\resizebox{0.47\textwidth}{!}{\begin{tikzpicture}
\tikzstyle{main}=[circle, minimum size = 5mm, thick, draw =black!80, node distance = 7mm]
\tikzstyle{connect}=[-latex, thick]
  \node[main, fill=red!20] [xshift=2cm, yshift=1cm] (PA) {$A_1$};
  \node[main, fill=blue!20] (PB) [right=of PA, yshift=-1cm] {$B_1$};
  \node[main, fill=magenta!60] (PC) [left=of PB, yshift=-1cm] {$C_1$};
 \node[main, fill=cyan] (PD) [left=of PA, yshift=-1cm] {$D_1$};
 
 \node[main, fill=red!20] [right=of PA, xshift=2.7cm] (PA2) {$A_2$};
  \node[main, fill=blue!20] (PB2) [right=of PA2, yshift=-1cm] {$B_2$};
  \node[main, fill=magenta!60] (PC2) [left=of PB2, yshift=-1cm] {$C_2$};
 \node[main, fill=cyan] (PD2) [left=of PA2, yshift=-1cm] {$D_2$};

\draw [ultra thick] (PA) -- (PB) node [very near start, xshift=0.2cm, yshift=1pt] {1};
 \draw [ultra thick] (PA) -- (PB) node [very near end, yshift=0.2cm] {1};
 
 \draw [ultra thick] (PA2) -- (PB2) node [very near start, xshift=0.2cm, yshift=1pt] {1};
 \draw [ultra thick] (PA2) -- (PB2) node [very near end, yshift=0.2cm] {1};
 
 \draw [ultra thick] (PB) -- (PC) node [very near start, xshift=-0.2cm, yshift=2pt] {2};
 \draw [ultra thick] (PB) -- (PC) node [very near end, yshift=0.2cm] {2};
 
 \draw [ultra thick] (PB2) -- (PC2) node [very near start, xshift=-0.2cm, yshift=2pt] {2};
 \draw [ultra thick] (PB2) -- (PC2) node [very near end, yshift=0.2cm] {2};
 
 \draw [ultra thick] (PC) -- (PD) node [very near start, xshift=0.1cm, yshift=4pt] {1};
 \draw [ultra thick] (PC) -- (PD) node [very near end, yshift=0.1cm, xshift=0.1cm] {1};
 
 \draw [ultra thick] (PC2) -- (PD2) node [very near start, xshift=0.1cm, yshift=4pt] {1};
 \draw [ultra thick] (PC2) -- (PD2) node [very near end, yshift=0.1cm, xshift=0.1cm] {1};

\draw [ultra thick] (PD) -- (PA) node [very near start, xshift=-0.1cm, yshift=5pt] {2};
 \draw [ultra thick] (PD) -- (PA) node [very near end, xshift=-0.2cm, yshift=0.1cm] {2};
 
 \draw [ultra thick] (PD2) -- (PA2) node [very near start, xshift=-0.1cm, yshift=5pt] {2};
 \draw [ultra thick] (PD2) -- (PA2) node [very near end, xshift=-0.2cm, yshift=0.1cm] {2};
 
 \draw [ultra thick] (PD) -- (PB) node [very near start, xshift=-0.1cm, yshift=5pt] {3};
 \draw [ultra thick] (PD) -- (PB) node [very near end, xshift=-0.05cm, yshift=0.2cm] {3};
 
  \draw [ultra thick] (PD2) -- (PB2) node [very near start, xshift=-0.1cm, yshift=5pt] {3};
 \draw [ultra thick] (PD2) -- (PB2) node [very near end, xshift=-0.05cm, yshift=0.2cm] {3};

   
   \node[main, fill=red!20] [xshift=6cm, yshift=0.5cm, left=of PA2] (RA) {$\underline{A}_1$};
  \node[main, fill=blue!20] (RB) [right=of RA, yshift=-1cm] {$\underline{B}_1$};
  \node[main, fill=magenta!60] (RC) [left=of RB] {$\underline{C}_1$};
 \node[main, fill=cyan] (RD) [left=of RA, yshift=-1cm] {$\underline{D}_1$};

\draw [ultra thick] (RA) -- (RB) node [very near start, xshift=0.2cm, yshift=1pt] {1};
 \draw [ultra thick] (RA) -- (RB) node [very near end, yshift=0.2cm] {1};
 
 \draw [ultra thick] (RC) -- (RD) node [very near start, xshift=-0.1cm, yshift=0.2cm] {1};
 \draw [ultra thick] (RC) -- (RD) node [near end, yshift=0.2cm] {1};
 
 \draw [ultra thick] (RB) -- (RC) node [very near start, xshift=-0.1cm, yshift=0.2cm] {2};
 \draw [ultra thick] (RB) -- (RC) node [very near end, yshift=0.2cm] {2};
 
 \draw [ultra thick] (RD) -- (RA) node [very near start, xshift=-0.1cm, yshift=0.2cm] {2};
 \draw [ultra thick] (RD) -- (RA) node [very near end, yshift=0.2cm, xshift=-0.1cm] {2};

 \foreach \from/\to in {RA/RB, RB/RC, RC/RD, RD/RA}
 \draw [ultra thick] (\from) -- (\to); 
  
 \node[rectangle,draw=white!100, below=of PA, yshift=0.9cm] {\bm{$G_1$}};
 \node[rectangle,draw=white!100, below=of PA2, yshift=0.9cm] {\bm{$G_1$}};

  \node[main, fill=red!20] [below=of RC, yshift=0.5cm] (BA) {$\underline{A}_2$};
  \node[main, fill=cyan] (BD) [right=of BA, yshift=-1cm] {$\underline{D}_2$};
  \node[main, fill=magenta!60] (BC) [left=of BD] {$\underline{C}_2$};
 \node[main, fill=blue!20] (BB) [left=of BA, yshift=-1cm] {$\underline{B}_2$};

\draw [ultra thick] (BA) -- (BD) node [very near start, xshift=0.2cm, yshift=1pt] {1};
 \draw [ultra thick] (BA) -- (BD) node [very near end, yshift=0.2cm] {1};
 
 \draw [ultra thick] (BC) -- (BB) node [very near start, xshift=-0.1cm, yshift=0.2cm] {1};
 \draw [ultra thick] (BC) -- (BB) node [near end, yshift=0.2cm] {1};
 
 \draw [ultra thick] (BD) -- (BC) node [very near start, xshift=-0.1cm, yshift=0.2cm] {2};
 \draw [ultra thick] (BD) -- (BC) node [very near end, yshift=0.2cm] {2};
 
 \draw [ultra thick] (BB) -- (BA) node [very near start, xshift=-0.1cm, yshift=0.2cm] {2};
 \draw [ultra thick] (BB) -- (BA) node [very near end, yshift=0.2cm, xshift=-0.1cm] {2};
 
 \draw [ultra thick] (BB) -- (RD) node [very near start, xshift=-0.2cm] {3};
 \draw [ultra thick] (BB) -- (RD) node [very near end, yshift=-0.1cm, xshift=-0.2cm] {3};
 
 \draw [ultra thick] (BD) -- (RB) node [very near start, xshift=-0.2cm] {3};
 \draw [ultra thick] (BD) -- (RB) node [very near end, yshift=-0.1cm, xshift=-0.2cm] {3};

\node[rectangle,draw=white!100, below=of RD, yshift=1cm, xshift=0.8cm] {\bm{$G_2$}};
  

\end{tikzpicture}}
\end{figure}

\noindent We craft a separate construction for 
the remaining properties, namely, $k$-clique and conjoint cycle.
The main idea is to replicate the effect of the common edge in the conjoint cycle via two identical components of another graph (that does not have any conjoint cycle) such that the components are cleverly aligned to reproduce the local port-ordered neighborhoods and thus present the same view to each node (see the adjoining figure).
Specifically, each conjoint cycle is denoted by $G_1$, and the other graph that does not have any conjoint cycles by $G_2$. The graphs, being port-locally isomorphic, are indistinguishable by CPNGNN.\\

\noindent For the $k$-clique, we simply connect 
$A_1$ to $C_1$, $A_2$ to $C_2$, $\underline{A}_1$ to $\underline{C}_1$, and $\underline{A}_2$ to $\underline{C}_2$ via a new port 3 at each of these nodes. Doing so ensures that the new graphs are port-locally isomorphic as well. Adding these edges, we note that, unlike $G_2$, each conjoint cycle $G_1$ yields a $4$-clique.   
\end{proof}

\noindent \textbf{Proof of Proposition \ref{propDime1}}
 \begin{proof}
 We now demonstrate the representational limits of DimeNets. Specifically, we show two graphs that differ in several graph properties such as girth, circumference, diameter,  radius, or total number of cycles. However, these graphs cannot be distinguished by DimeNets.  \\ 
 
\noindent Note that DimeNet will be able to discriminate $S_8$ from the graph with cycles $S_4$ (recall our construction in Proposition \ref{prop1}), since, e.g., $\angle B_1C_1D_1$ in $S_4$ is different from $\angle \underline{B}_1\underline{C}_1\underline{D}_1$ in $S_8$. In order to design a failure case for DimeNet, we need to construct a pair of non-isomorphic graphs that
have not only identical local pairwise distances but also angles, so that their output embedding is same.

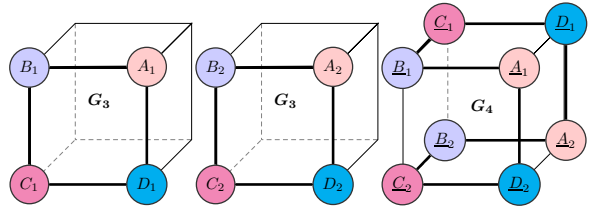
\begin{figure}[h]
\centering
\resizebox{0.47\textwidth}{!}{\begin{tikzpicture}
\tikzstyle{main}=[circle, minimum size = 3mm, thick, draw =black!80]
  \pic at (1,-3) {annotated cuboid={width=250, height=250, depth=250, scale=.01, units=}};
  \pic at (5,-3) {annotated cuboid={width=250, height=250, depth=250, scale=.01, units=}};
   \pic at (9,-3) {annotated cuboid={width=250, height=250, depth=250, scale=.01, units=}};
   
 \node[main, fill=cyan] (PD)[yshift=-5.5cm, xshift=1cm] {$D_1$};
 \node[main, fill=red!20]  (PA)[above=of PD, yshift=0.65cm] {$A_1$};
 \node[main,fill=magenta!60]  (PC)[left=of PD, xshift=-0.65cm] {$C_1$};
\node[main, fill=blue!20] (PB) [above=of PC, yshift=0.65cm] {$B_1$};
\foreach \from/\to in {PA/PB, PB/PC, PC/PD, PD/PA}
 \draw[ultra thick] (\from) -- (\to); 
 
 \node[main, fill=cyan] (RD)[yshift=-5.5cm, xshift=5cm] {$D_2$};
 \node[main, fill=red!20]  (RA)[above=of RD, yshift=0.65cm] {$A_2$};
 \node[main,fill=magenta!60]  (RC)[left=of RD, xshift=-0.65cm] {$C_2$};
\node[main, fill=blue!20] (RB) [above=of RC, yshift=0.65cm] {$B_2$};
\foreach \from/\to in {RA/RB, RB/RC, RC/RD, RD/RA}
 \draw[ultra thick] (\from) -- (\to);

\node[main, fill=cyan] (PD2)[yshift=-5.5cm, xshift=9cm] {$\underline{D}_2$}; 
\node[main, fill=red!20]  (PA1)[above=of PD2, yshift=0.6cm] {$\underline{A}_1$};
 \node[main,fill=magenta!60]  (PC2)[left=of PD2, xshift=-0.6cm] {$\underline{C}_2$};
 \node[main, fill=blue!20] (PB1) [above=of PC2, yshift=0.6cm] {$\underline{B}_1$};
 \foreach \from/\to in {PA1/PB1,  PC2/PD2, PD2/PA1}
 \draw[ultra thick] (\from) -- (\to); 
 
 \node[main, fill=red!20] (PA2)[yshift=-4.55cm, xshift=10cm] {$\underline{A}_2$}; 
 \node[main, fill=blue!20] (PB2) [left=of PA2, xshift=-0.7cm] {$\underline{B}_2$};
 \node[main, fill=cyan] (PD1) [above=of PA2, yshift=0.6cm] {$\underline{D}_1$};
 \node[main,fill=magenta!60] (PC1) [left=of PD1, xshift=-0.7cm] {$\underline{C}_1$};
 \foreach \from/\to in {  PB1/PC1, PC1/PD1, PB2/PA2, PD1/PA2, PC2/PB2}
 \draw[ultra thick] (\from) -- (\to); 

\node[draw=white!100, above=of PC, yshift=0.1cm, xshift=1.5cm] {\bm{$G_3$}};

\node[draw=white!100, above=of RC, yshift=0.1cm, xshift=1.5cm] {\bm{$G_3$}};

\node[draw=white!100, above=of PB2, yshift=-1cm, xshift=0.8cm] {\bm{$G_4$}};

\end{tikzpicture}}
\end{figure}

\noindent Our idea is to overlay the cycles $S_4$ and $S_8$ on a cube (see $G_3$ and $G_4$ - the graphs consist of only edges in bold). Doing so does not have any bearing on the graph properties.
Since we orient the edges of these cycles along the sides of the cube, the local distances are identical. Moreover, by having $A_1B_1C_1D_1$ and $A_2B_2C_2D_2$ as opposite faces of the cube, we ensure that each angle in $G_4$ is a right angle, exactly as in $G_3$. Thus, for each $X \in \{A, B, C, D\}$, nodes $X_1$, $X_2$, $\underline{X}_1$, and $\underline{X}_2$  have identical feature vectors and identical local spatial information. Thus, the embeddings for $X_1$, $X_2$, $\underline{X}_1$, and $\underline{X}_2$ are identical, and any permutation-invariant readout results in identical output embeddings for the two graphs. \\
\end{proof}

\noindent \textbf{Proof of Proposition \ref{Proposition4}}
\begin{proof}
We now show that the complexity of the GNN may be bounded by the complexity of the computation trees. In other words, the worst case generalization bound over a set of graphs corresponds to having each graph be a single computation tree.\\

\noindent Formally,
\begin{eqnarray*} \hat{\mathcal{R}}_\mathcal{G} & \triangleq &  \mathbb{E}_{\sigma} \sup_{\Theta} \sum_{j=1}^m \sigma_j f(G_j;\Theta)\\
 &  = & \mathbb{E}_{\sigma} \sup_{\Theta}  \sum_{j=1}^m \sigma_j \mathbb{E}_{T \sim w'(G_j)} f_c(T;\Theta) \\
 &  \leq & \mathbb{E}_{\sigma} \mathbb{E}_{t_1,..t_m} \sup_{\Theta}  \sum_{j=1}^m \sigma_j f_c(t_j;\Theta)~\\
 & = & \mathbb{E}_{t_1,..t_m} \underbrace{\mathbb{E}_{\sigma} \sup_{\Theta}  \sum_{j=1}^m \sigma_j f_c(t_j;\Theta)}_{\hat{\mathcal{R}}_{\mathcal{T}}}~,
\end{eqnarray*}
where we invoked Jensen's inequality to swap the expectation with supremum for our inequality (the operation is permissible since $\sup$ is a convex function). 
\end{proof}
\noindent \textbf{Proof of Lemma \ref{Lemma1}}
\begin{proof}
Our objective here is to bound the 
effect of change in weights from $(W_1, W_2)$ to $(W_1', W_2')$ on the embedding of the root node of our fixed tree (that has depth $L$). Since non-linear activation and permutation-invariant aggregation are both Lipschitz-continuous functions, and the feature vector at the root $x_L$  and the weights have bounded norm, the embedding at the root of the tree adapts to the embeddings from the subtrees. \\

\noindent Specifically, we note that the $l_2$-norm of difference of embedding vectors produced by $(W_1, W_2)$ and $(W_1', W_2')$ is 
\small
\begin{eqnarray}
\Delta_L & \triangleq &\left|\right|T_L(W_1, W_2) - T_L(W_1', W_2')||_2 \nonumber \\
&=&\bigg|\bigg|\phi\big(W_1 x_L + W_2\underbrace{ \rho\big(\sum_{j \in C(x_{L})} g(T_{L-1,j} (W_1, W_2)\big)}_{\triangleq ~R(W_1, W_2, x_L)}\big)\big) \nonumber \\
&-& \phi\big(W_1' x_L + W_2' \rho\big(\sum_{j \in C(x_{L})} g(T_{L-1,j} (W_1', W_2')\big)\big)\big)
\bigg|\bigg|_2  \nonumber \\
&\leq&C_{\phi}\left|\left|(W_1-W_1')x_L\right|\right|_2 \label{eq1} \\ & \quad+ & C_{\phi}||W_2 R(W_1, W_2, x_L) - W_2' R(W_1', W_2', x_L)||_2 \nonumber~.
\end{eqnarray}
\normalsize
Therefore, in order to find an upper bound for $\Delta_L$, we will bound the two terms in the last inequality separately. We first bound the second term using the sum of $||W_2 R(W_1, W_2, x_L) - W_2'R(W_1, W_2, x_L)||_2$ and $||W_2'R(W_1, W_2, x_L) - W_2'R(W_1', W_2', x_L)||_2$.
Note that
\begin{eqnarray}
&& ||W_2'R(W_1, W_2, x_L) - W_2'R(W_1', W_2', x_L)||_2 \nonumber \\
&\leq& ||W_2'||_2 ~~||R(W_1, W_2, x_L) - R(W_1', W_2', x_L)||_2~.  \nonumber\\ \label{eq2}
\end{eqnarray}
Since $g$ is $C_g$-Lipschitz, the branching factor of tree is $d$, and $\rho$ is $C_{\rho}$-Lipschitz, therefore, $R$ is $dC_gC_{\rho}$-Lipschitz.  We will use this fact to bound \eqref{eq2}. Specifically,
\begin{eqnarray*} &&\hspace*{-1cm} ||R(W_1, W_2, x_L) - R(W_1', W_2', x_L)||_2\\
&\leq& C_{\rho} \bigg|\bigg|\sum_{j \in C(x_{L})} \bigg( g(T_{L-1,j} (W_1, W_2))\\ 
&&\hspace*{2cm}- g(T_{L-1,j} (W_1', W_2')) \bigg)\bigg|\bigg|_2\\
&\leq& C_{\rho} \sum_{j \in C(x_{L})}  \bigg|\bigg|\bigg( g(T_{L-1,j} (W_1, W_2))\\ 
&&\hspace*{2cm}- g(T_{L-1,j} (W_1', W_2')) \bigg)\bigg|\bigg|_2\\
&\leq& C_{\rho} C_g \sum_{j \in C(x_{L})}  \bigg|\bigg|T_{L-1,j} (W_1, W_2)
- T_{L-1,j} (W_1', W_2')\bigg|\bigg|_2\\ 
&=& C_{\rho} C_g \sum_{j \in C(x_{L})} \Delta_{L-1, j}~.
\end{eqnarray*}
Using this with $||W_2'||_2 \leq B_{2}$ in \eqref{eq2}, we immediately get 
\begin{eqnarray*} 
||W_2'R(W_1, W_2, x_L)&-&W_2'R(W_1', W_2', x_L)||_2 ~\nonumber\\
 &\leq& B_{2}C_{\rho} C_g \sum_{j \in C(x_{L})} \Delta_{L-1, j}~\nonumber \\
 &\leq& B_{2}C_{\rho} C_g d \max_{j \in C(x_{L})} \Delta_{L-1, j}~. \label{eq3}
\end{eqnarray*}

\noindent In other words, we bound the effect on each subtree of the root by the maximum effect across these subtrees. Combining this with $||x_L||_2 \leq B_x$, we note from \eqref{eq1} that
\begin{eqnarray}
\Delta_L &\leq&  C_{\phi}B_x\left|\left|(W_1-W_1')\right|\right|_2 \nonumber \\&~~~+& C_{\phi}B_{2}C_{\rho} C_g d \max_{j \in C(x_{L})} \Delta_{L-1, j} \label{eq4} \\ &+& C_{\phi}||(W_2-W_2') R(W_1, W_2, x_L)||_2 \nonumber~. \end{eqnarray} \end{proof}

\noindent \textbf{Proof of Lemma \ref{Lemma2}}
\begin{proof}
Note from \eqref{eq4} that in order for the change in embedding of the root (due to a small change in weights) to be small, we require that the last term in  \eqref{eq4} is small.  Toward, that goal we bound the norm of permutation-invariant aggregation at the root node. Specifically, we note that \\
\begin{eqnarray}
&& ||R(W_1, W_2, x_L)||_2 \nonumber \\ & = & \bigg|\bigg|\rho\big(\sum_{j \in C(x_{L})} g(T_{L-1,j} (W_1, W_2)\big)\big)\bigg|\bigg|_2 \nonumber\\
& = & \bigg|\bigg|\rho\big(\sum_{j \in C(x_{L})} g(T_{L-1,j} (W_1, W_2)\big)\big) - \rho(0) \bigg|\bigg|_2 \nonumber \\
&\leq& C_{\rho} \bigg|\bigg| \sum_{j \in C(x_{L})} g(T_{L-1,j} (W_1, W_2))\bigg|\bigg|_2 \nonumber \\
&\leq& C_{\rho} \sum_{j \in C(x_{L})} \bigg|\bigg| g(T_{L-1,j} (W_1, W_2)) - g(0)\bigg|\bigg|_2 \nonumber \\
&\leq& C_{\rho}C_g \sum_{j \in C(x_{L})} \bigg| \bigg| T_{L-1,j} (W_1, W_2) \bigg| \bigg|_2~. \nonumber \\
&\leq& C_{\rho}C_g d \max_{j \in C(x_{L})} \bigg| \bigg| T_{L-1,j} (W_1, W_2) \bigg| \bigg|_2~,
\label{eq5}
\end{eqnarray}
where the norm of the embedding produced by children $j$ of the root using weights $W_1$ and $W_2$ is given by\\
$\bigg| \bigg| T_{L-1,j} (W_1, W_2) \bigg| \bigg|_2$
\begin{eqnarray}
&=& \bigg|\bigg|\phi(W_1x_{L-1, j} + W_2 R(W_1, W_2, x_{L-1, j}))\bigg|\bigg|_2 \nonumber \\
&=& \bigg|\bigg|\phi(W_1x_{L-1, j} + W_2 R(W_1, W_2, x_{L-1, j})) - \phi(0)\bigg|\bigg|_2 \nonumber \\
&\leq& C_{\phi} \bigg|\bigg|W_1x_{L-1, j} + W_2 R(W_1, W_2, x_{L-1, j})\bigg|\bigg|_2 \nonumber \\
&\leq& C_{\phi}\bigg|\bigg|W_1x_{L-1, j}\bigg|\bigg|_2 + C_{\phi}\bigg|\bigg|W_2 R(W_1, W_2, x_{L-1, j})\bigg|\bigg|_2 \nonumber \\
&\leq& C_{\phi}B_{1}B_x + C_{\phi}B_{2} \bigg|\bigg| R(W_1, W_2, x_{L-1, j})\bigg|\bigg|_2~. \label{eq6}
\end{eqnarray}
Also, since $||\phi(x)||_{\infty} \leq b$ for all $x \in \mathbb{R}^r$ (by our assumption), and $||\phi(x)||_2 \leq \sqrt{r} ||\phi(x)||_{\infty}$, we obtain 
\begin{equation} \label{eq7}
\bigg| \bigg| T_{L-1,j} (W_1, W_2) \bigg| \bigg|_2 ~\leq~ b \sqrt{r}~.
\end{equation}
Combining \eqref{eq5} and \eqref{eq6}, we get the recursive relationship
\begin{eqnarray}
&&||R(W_1, W_2, x_L)||_2 \nonumber \\
&\leq& C_{\rho}C_{g}C_{\phi}B_{1}B_x d \nonumber \\
&+& C_{\rho}C_{g}C_{\phi}B_{2}d ~\max_{j \in C(x_L)} \bigg|\bigg| R(W_1, W_2, x_{L-1, j})\bigg|\bigg|_2~ \nonumber \\
&\leq& C_{\rho}C_{g}C_{\phi}B_{1}B_x d \sum_{\ell=0}^{L-1} (C_{\rho}C_{g}C_{\phi}B_{2} d)^{\ell} \nonumber \\
&=& C_{\rho}C_{g}C_{\phi}B_{1}B_x d~  \dfrac{(\mathcal{C} d)^{L}-1}{\mathcal{C} d-1}~. \label{eq8}
\end{eqnarray}
On the other hand, combining \eqref{eq5} and \eqref{eq6}, we get
\begin{equation} \label{eq9}
||R(W_1, W_2, x_L)||_2 ~\leq~ bdC_{\rho}C_g\sqrt{r}~.
\end{equation}
Taken together, \eqref{eq8} and \eqref{eq9} yield $||R(W_1, W_2, x_L)||_2$
\begin{equation} \label{eq10}
 ~\leq~ C_\rho C_g d \min\left\{b\sqrt{r}, C_{\phi}B_{1}B_x  \dfrac{(\mathcal{C} d)^{L}-1}{\mathcal{C} d-1}\right\}~.  
\end{equation}
\end{proof}

\noindent \textbf{Proof of Lemma \ref{Lemma3}}
\begin{proof}
Using the results from Lemma \ref{Lemma1} and \ref{Lemma2}, we will simplify the bound on $\Delta_L$, i.e., the change in embedding due to a change in weights. We will then bound the change in probability (that the tree label is 1) $\Lambda_L$ in terms of $\Delta_L$, when we change not only the weights from $(W_1, W_2)$ to $W_1', W_2'$ but also the local classifier parameters from $\beta$ to $\beta'$ (where $\beta$ and $\beta'$ are chosen from a bounded norm family). We show these steps below. \\

\noindent Plugging the bound on 
$\overline{R} \triangleq ||R(W_1, W_2, x_L)||_2$ from Lemma \ref{Lemma2} in Lemma \ref{Lemma1}, we get
\begin{eqnarray*}
\Delta_L &\leq&  C_{\phi}B_x\left|\left|W_1-W_1'\right|\right|_2 \nonumber \\&~~~+& \mathcal{C} d \max_{j \in C(x_{L})} \Delta_{L-1, j} \\ &+& C_{\phi}||W_2-W_2'||_2 \overline{R} \nonumber~. 
\end{eqnarray*}
Expanding the recursion, we note that 
\begin{equation} \label{eq11}
 \Delta_L \leq  M B_x\left|\left|W_1-W_1'\right|\right|_2 + M  \overline{R} ||W_2-W_2'||_2~,
\end{equation}
where 
\begin{eqnarray} M = C_{\phi} \dfrac{\left(\mathcal{C} d\right)^L - 1}{\mathcal{C} d - 1}~. 
\end{eqnarray}
Since $||A||_2 \leq ||A||_F$ for every matrix $A$, we have
\begin{equation} \label{eq13}
 \Delta_L \leq  M B_x\left|\left|W_1-W_1'\right|\right|_F + M  \overline{R} ||W_2-W_2'||_F~.
\end{equation}

Now since sigmoid is $1$-Lipschitz, we have 
\begin{eqnarray}
\Lambda_L &=& |\psi(\beta^{\top} T_L(W_1, W_2)) - \psi( {\beta'}^{\top} T_L(W_1', W_2'))| \nonumber \\
&\leq& |\beta^{\top} T_L(W_1, W_2) - {\beta'}^{\top} T_L(W_1, W_2)| \nonumber \\
&+& |{\beta'}^{\top} T_L(W_1, W_2) - {\beta'}^{\top} T_L(W_1', W_2')| \nonumber \\
&\leq& ||\beta-\beta'||_2 ~||T_L(W_1, W_2)||_2 + B_{\beta} \Delta_L \nonumber\\
&\leq& ||\beta-\beta'||_2 \underbrace{(C_{\phi}B_{1}B_x + C_{\phi}B_{2}\overline{R})}_{Z} + B_{\beta} \Delta_L~ \nonumber\\ \label{eq14}
\end{eqnarray}
using  \eqref{eq6} and \eqref{eq10}.
\end{proof}

\noindent \textbf{Proof of Lemma \ref{Lemma4}}
\begin{proof}
Building on results from Lemmas \ref{Lemma1}-\ref{Lemma3}, we will now show that the change in probability $\Lambda_L$ can be bounded by $\epsilon$, using a covering of size $P$, where $P$ depends on $\epsilon$. Moreover, we show that $\log P$ grows as $\mathcal{O}\left(\log \left(\dfrac{1}{\epsilon}\right)\right)$ for sufficiently small values of $\epsilon$. That is, we can ensure $\Lambda_L$ is small by using a small covering. 

\noindent We begin by noting that we can find a covering $\mathcal{C}\left(\beta, \dfrac{\epsilon}{3Z_{\ell}}, ||\cdot||_2\right)$ of size 
\begin{equation*} \mathcal{N}\left(\beta, \dfrac{\epsilon}{3Z_{\ell}}, ||\cdot||_2\right)  ~\leq~ \left(1 + \dfrac{6Z B_{\beta}}{\epsilon}    \right)^{r}~.
\end{equation*}
Thus, for any specified $\epsilon$, we can ensure that $\Lambda_L$ is at most $\epsilon$ by finding matrix coverings $\mathcal{C}\left(W_1, \dfrac{\epsilon}{3M B_x B_{\beta}}, ||\cdot||_F\right)$ and 
$\mathcal{C}\left(W_2, \dfrac{\epsilon}{3M \overline{R} B_{\beta}}, ||\cdot||_F\right)$. Using  Lemma 8 from \citep{CLZ2019}, we obtain the corresponding bounds on their covering number. Specifically,
\small
\begin{equation*} \mathcal{N}\left( W_1, \dfrac{\epsilon}{3M B_x B_{\beta}}, ||\cdot||_F \right)  ~\leq~ \left(1 + \dfrac{6M B_x B_{\beta} B_{1}\sqrt{r}}{\epsilon}    \right)^{r^2},
\end{equation*}
\begin{equation*} \mathcal{N}\left(W_2, \dfrac{\epsilon}{3M \overline{R}B_{\beta}}, ||\cdot||_F \right)  ~\leq~ \left(1 + \dfrac{6M \overline{R} B_{\beta} B_{2}\sqrt{r}}{\epsilon} \right)^{r^2}.
\end{equation*}
\normalsize
The product of all the covering numbers is bounded by
\small
\begin{eqnarray*}
P = \left(1 + \dfrac{6B_{\beta} \max\{Z, M \sqrt{r} \max\{B_x B_{1}, \overline{R} B_{2}\}\}}{\epsilon} \right)^{2r^2+r}~. 
\end{eqnarray*}
\normalsize
Therefore, the class $\mathcal{B} (L, d, r, \beta, B_1, B_2, B_x)$ that maps a tree-structured input to the probability that the corresponding tree label is $1$ 
can be approximated to within $\epsilon$ by a covering of size $P$. Moreover, when 
$$\epsilon < 6B_{\beta} \max\{Z, M \sqrt{r} \max\{B_x B_{1}, \overline{R} B_{2}\}\},$$ we obtain that $\log P$ is at most
$$3r^2 \log\left(\dfrac{12B_{\beta} \max\{Z, M \sqrt{r} \max\{B_x B_{1}, \overline{R} B_{2}\}\}}{\epsilon}\right)~.$$
\end{proof}

\noindent \textbf{Proof of Proposition \ref{PropGen}}
\begin{proof}
We are now ready to prove our  generalization bound.  Specifically, we  invoke a specific form of Dudley's entropy integral to bound the empirical Rademacher complexity $\hat{\mathcal{R}}_{\mathcal{T}}(\mathcal{J}_\gamma)$ via our result on covering from Lemma \ref{Lemma4}, where recall that $\mathcal{J}_\gamma$ maps each tree-label pair $(t, y)$  to margin loss $\loss_\gamma(-p(f_c(t; \Theta),y))$.\\ 


\noindent It is straightforward to show that $p$ is 2-Lipschitz in its first argument, and $loss_{\gamma}$ is $\dfrac{1}{\gamma}$-Lipschitz.  Therefore, we can approximate  the class $\mathcal{I}$ that maps $(t, y)$ to 
$p(f_c(t; \Theta), y)$
by finding an $\dfrac{\epsilon}{2}$-cover of $\mathcal{B}$. 
Now, note that $\mathcal{I}$ takes values in the interval 
$[-e, e]$, where 
\begin{eqnarray*}
e  =  ||u||_2~||T_L(W_1, W_2)||_2  ~\leq~  B_{\beta} Z~.
\end{eqnarray*}
Using Lemma A.5. in \citep{BFT2017}, we obtain that
\begin{eqnarray*}
\hat{\mathcal{R}}_\mathcal{T}(\mathcal{I}) \leq \inf_{\alpha > 0} \left(\dfrac{4 \alpha}{\sqrt{m}} + \dfrac{12}{m} \int_{\alpha}^{2 e \sqrt{m}} \sqrt{\log  \mathcal{N}(\mathcal{I}, \epsilon, ||\cdot||)} d \epsilon  \right)
\end{eqnarray*}
where, using Lemma \ref{Lemma4}, we have
\begin{eqnarray*}
\int_{\alpha}^{2 e \sqrt{m}} \sqrt{\log  \mathcal{N}(\mathcal{I}, \epsilon, ||\cdot||)} d \epsilon  \\  \leq \int_{\alpha}^{2 e \sqrt{m}} \sqrt{\log  \mathcal{N}(\mathcal{B}, \dfrac{\epsilon}{2}, dist(\cdot, \cdot))} d \epsilon 
\\ \leq \int_{\alpha}^{2 e \sqrt{m}} \sqrt{\log U} ~~~ \leq 2e \sqrt{m} \sqrt{\log U} ~~~ = 2 B_{\beta} Z \sqrt{m \log U}
\end{eqnarray*}
with $dist$ being the combination of $||\cdot||_2$ and $||\cdot||_F$ norms used to obtain covering of size $P$ in Lemma \ref{Lemma4}, and $\log U$ is 
$$3r^2 \log \left(\dfrac{24B_{\beta} \max\{Z, M \sqrt{r} \max\{B_x B_{1}, \overline{R} B_{2}\}\}}{\alpha}\right)~.$$
Setting $\alpha = \sqrt{\dfrac{1}{m}}$, we immediately get
$$\hat{\mathcal{R}}_{\mathcal{T}}(\mathcal{I}) \leq \dfrac{4}{m} + \dfrac{24 B_{\beta} Z}{\sqrt{m}} \sqrt{3r^2 \log Q}~, $$
where $$Q = 24B_{\beta} \sqrt{m} \max\{Z, M \sqrt{r} \max\{B_x B_{1}, \overline{R} B_{2}\}\}~.$$

\noindent We finally bound the complexity of class $\hat{\mathcal{R}}_{\mathcal{T}}(\mathcal{J}_{\gamma})$ by noting that $loss_{\gamma}$ is $\dfrac{1}{\gamma}$-Lipschitz, and invoking Talagrand's lemma \citep{MRT2012}:

\begin{eqnarray*}
\hat{\mathcal{R}}_{\mathcal{T}}(\mathcal{J}_{\gamma}) \leq  \dfrac{\hat{\mathcal{R}}_{\mathcal{T}}(\mathcal{I})}{\gamma} \leq \dfrac{4}{\gamma m} + \dfrac{24 r B_{\beta} Z }{\gamma \sqrt{m}} \sqrt{3\log Q}~. 
\end{eqnarray*}
\end{proof}
\noindent \textbf{Proof of Proposition \ref{Injective}} 
\begin{proof}
We first convey some intuition. Suppose $|X| < 8$, and we assign a distinct index $z(x) \in \{1, 2, \ldots, 8\}$ to each message $x \in X$. Then, we can map each $x$ to $10^{-z(x)}$, i.e., obtain a decimal expansion which may be viewed as a one-hot vector representation of at most 10 digits. We would reserve a separate block of 10 digits for each port.
This would allow us to disentangle the coupling between messages and their corresponding ports. 
Specifically, since the ports are all distinct, we can {\em shift} the digits in expansion of $x$ to the right by dividing by $10^{p}$, where $p$ is the port number of $x$. This allows us to represent each $(x, p)$ pair uniquely. \\

\noindent Formally, since $\mathcal{X}$ is countable, there exists a mapping $Z : \mathcal{X} \mapsto \mathbb{N}$ from $x \in \mathcal{X}$ to natural numbers. Since $X$ has bounded cardinality, we know the existence of some $N \in \mathbb{N}$ such that $|X| < N$ for all $X$. Define $k = 10^{\lceil \log_{10} N \rceil}$. We define function $f$ in the proposition as $f(x) = k^{-Z(x)}$. We also take function $g$ in proposition to be $g(p) = 10^{-kN(p-1)}$. That is, we express the function $h$ as $h((x_1, p_1), \ldots, (x_{|P|}, p_{|P|})) = \sum_{i=1}^{|P|} g(p_i) f(x_i)$.  
\end{proof}

\end{document}